\documentclass{new_tlp}

\usepackage{mathptmx}
\usepackage{url}
\usepackage{subcaption}
\usepackage{xspace}
\usepackage{amssymb,amsmath,amsfonts}
\usepackage{stmaryrd}
\usepackage{color,xcolor}
\usepackage{booktabs}
\usepackage{comment}
\usepackage{enumerate}
\usepackage[defblank]{paralist}
\usepackage{hyperref}

\usepackage{tikz}
\usetikzlibrary{matrix,positioning}

\definecolor{incolor}{RGB}{210,220,230}
\definecolor{outcolor}{RGB}{235,215,215}

\usepackage[thmmarks,amsmath]{ntheorem}

\usepackage{color,soul}

\theoremseparator{.}
\theorembodyfont{\normalfont}

\newtheorem{theorem}{Theorem}
\newtheorem{corollary}{Corollary}

\theoremstyle{plain}
\theoremsymbol{\ensuremath{\qedboxfull}}
\theorembodyfont{\normalfont}

\newtheorem{example}{Example}

\newcommand{\I}{I}

\newcommand{\ra}{\rightarrow}

\newcommand{\lora}{\longrightarrow}

\newcommand{\set}[1]{\{#1\}}                      % set

\newcommand{\card}[1]{|{#1}|}                     % cardinality of a set
\newcommand{\tup}[1]{\langle #1\rangle}            % tuple

\newcommand{\limp}{\rightarrow}

\newcommand{\Mod}[1]{\mathit{Mod}}

\newcommand{\qedboxfull}{\vrule height 5pt width 5pt depth 0pt}

\newcommand{\types}{\ensuremath{\mathfrak{D}}\xspace}
\newcommand{\fold}{\ensuremath{\mathsf{FOL}(\types)}\xspace}

\newcommand{\ALC}{\ensuremath{\mathsf{ALC}}\xspace}

\newcommand{\ALCd}{\ensuremath{\mathsf{ALC}(\mathit{D})}\xspace}
\newcommand{\ALCH}{\ensuremath{\mathsf{ALCH}}\xspace}
\newcommand{\ALCHdd}{\ensuremath{\ALCH(\types)}\xspace}

\makeatletter
\newcommand{\dlliteb}[1][\@empty]{\textit{DL-Lite}\ensuremath{_{\mathit{bool}}^{\ifx\@empty#1\else(\textit{#1})\fi}}\xspace}
\newcommand{\dlliteh}[1][\@empty]{\textit{DL-Lite}\ensuremath{_{\mathit{horn}}^{\ifx\@empty#1\else(\textit{#1})\fi}}\xspace}

\makeatother

\newcommand{\per}{\ldotp}

\newcommand{\AND}{\sqcap}
\newcommand{\OR}{\sqcup}
\newcommand{\NOT}{\neg}
\newcommand{\NOTnnf}{{\sim}}
\newcommand{\ALL}[2]{\forall #1 \per #2}
\newcommand{\SOME}[2]{\exists #1 \per #2}
\newcommand{\SOMET}[1]{\exists #1}

\newcommand{\UNDEF}[1]{#1\!\!\uparrow}

\newcommand{\ISA}{\sqsubseteq}

\newcommand{\dom}[1][\I]{\Delta^{#1}}  % \Delta^I
\newcommand{\Int}[2][\I]{#2^{#1}}      % #2^I    (interpretation function)
\newcommand{\INT}[2][\I]{(#2)^{#1}}    % (#2)^I  (interpretation function)
\newcommand{\inter}[1][\I]{\tup{\dom[#1],\Int[#1]{\cdot}}}  % <Delta^I,.^I>

\newcommand{\univ}{\Delta}

\newcommand{\TBox}{T}
\newcommand{\ABox}{A}

\newcommand{\none}{\cv{none}\xspace}
\newcommand{\indoor}{\cv{indoor}\xspace}
\newcommand{\outdoor}{\cv{outdoor}\xspace}

\newcommand{\dmn}{DMN\xspace}
\newcommand{\drg}{DRG\xspace}
\newcommand{\drgs}{\drg{s}\xspace}
\newcommand{\drd}{DRD\xspace}
\newcommand{\drds}{\drd{s}\xspace}
\newcommand{\feel}{FEEL\xspace}
\newcommand{\sfeel}{S-\feel}
\newcommand{\dkb}{DKB\xspace}
\newcommand{\dkbs}{DKBs\xspace}
\newcommand{\idkb}{IDKB\xspace}
\newcommand{\cbridge}{C}
\newcommand{\tfol}{\tau}
\newcommand{\tfolb}[1]{\tau_{#1}}

\newcommand{\dt}{\mathit{M}}
\newcommand{\cv}[1]{\ensuremath{\mathtt{#1}}}
\newcommand{\func}[1]{\ensuremath{\mathsf{#1}}}
\newcommand{\attr}[1]{\ensuremath{\mathbf{#1}}}

\newcommand{\tname}{\ensuremath{\mathit{Name}}}

\newcommand{\thit}{H}

\newcommand{\upol}{\cv{u}}
\newcommand{\apol}{\cv{a}}
\newcommand{\ppol}{\cv{p}}

\newcommand{\tin}{I}
\newcommand{\tout}{O}
\newcommand{\atype}{\func{AType}}

\newcommand{\infacet}{\func{InFacet}}
\newcommand{\outrange}{\func{ORange}}
\newcommand{\outdef}{\func{ODef}}

\newcommand{\higherp}{\prec}

\newcommand{\domain}{\Delta}
\newcommand{\sigp}{\Gamma}
\newcommand{\cond}{\varphi}
\newcommand{\anycond}{\mathtt{-}}
\newcommand{\orcond}{\mathtt{,}}
\newcommand{\trules}{{R}}
\newcommand{\incond}{\func{If}}
\newcommand{\outval}{\func{Then}}

\newcommand{\dts}{\mathfrak{M}}
\newcommand{\odts}{\mathfrak{M}_{out}}
\newcommand{\km}{\mathfrak{K}}
\newcommand{\ireq}{\rightarrow}

\newcommand{\adrg}{\mathfrak{G}}
\newcommand{\oimap}{\Rightarrow}
\newcommand{\oimaptrans}{\overset{*}{\Rightarrow}}
\newcommand{\freein}[1]{\func{FreeInputs}(#1)}
\newcommand{\battr}[1]{\func{BoundAttr}(#1)}

\newcommand{\adkb}{\mathfrak{X}}
\newcommand{\aidkb}{\overline{\adkb}}
\newcommand{\transl}[1]{#1'}
\newcommand{\name}[2]{#1\!\cdot\!\attr{#2}}

\makeatletter
\newcommand{\oset}[3][0ex]{%
  \mathrel{\mathop{#3}\limits^{
    \vbox to#1{\kern-2\ex@
    \hbox{$\scriptstyle#2$}\vss}}}}
\makeatother

\newcommand{\askmodels}{\oset[0ex]{?}{\models}}

\newcommand{\type}{D}
\newcommand{\dtype}{E}

\newcommand{\sig}{\Sigma}
\newcommand{\csig}{\sig_c}
\newcommand{\rsig}{\sig_r}
\newcommand{\fsig}{\sig_f}

\newcommand{\Sig}{\mathsf{Sig}}

\newcommand{\cl}[1]{\mathsf{cl}(#1)}

\newcommand{\rel}[1]{\textsl{#1}}

\newcommand{\get}[2]{#1.#2}

\newcommand{\exptime}{\textsc{ExpTime}\xspace}

\newcommand{\ACz}{AC$^0$\xspace}

\newcommand{\tdl}{\rho}
\newcommand{\tdlb}[1]{\rho_{#1}}

\usepackage{listings}

\title[Semantic DMN]{Semantic DMN: Formalizing and Reasoning About Decisions in
 the Presence of Background Knowledge}

\author[D.~Calvanese, M.~Dumas, F.~M.~Maggi, M.~Montali]{
 DIEGO CALVANESE and MARCO MONTALI\\
 Free University of Bozen-Bolzano, Italy\\
 \email{\{calvanese|montali\}@inf.unibz.it}
 \and
 MARLON DUMAS and FABRIZIO M. MAGGI\\
 University of Tartu, Estonia\\
 \email{\{m.dumas|f.m.maggi\}@ut.ee}
}

\jdate{September 2018}
\pubyear{2018}
\pagerange{\pageref{firstpage}--\pageref{lastpage}}
\doi{doi123}

\submitted{11 December 2017}
\revised{3 August 2018}
\accepted{24 August 2018}

\begin{document}

\label{firstpage}

\maketitle

\begin{abstract}
  The Decision Model and Notation (DMN) is a recent OMG standard for the
elicitation and representation of decision models, and for managing their
interconnection with business processes.  DMN builds on the notion of decision
tables, and their combination into more complex decision requirements graphs
(DRGs), which bridge between business process models and decision logic models.
DRGs may rely on additional, external business knowledge models, whose
functioning is not part of the standard.  In this work, we consider one of the
most important types of business knowledge, namely background knowledge that
conceptually accounts for the structural aspects of the domain of interest, and
propose \emph{decision knowledge bases} (DKBs), which semantically combine DRGs
modeled in DMN, and domain knowledge captured by means of first-order logic
with datatypes.  We provide a logic-based semantics for such an integration,
and formalize different DMN reasoning tasks for DKBs.  We then consider
background knowledge formulated as a description logic ontology with datatypes,
and show how the main verification tasks for DMN in this enriched setting can
be formalized as standard DL reasoning services, and actually carried out in
\exptime.  We discuss the effectiveness of our framework on a case study in
maritime security.

\end{abstract}


\tableofcontents

\section{Introduction}

The \emph{Decision Model and Notation} (\dmn) is a recent OMG standard for the
representation and enactment of decision models \cite{DMN}. The standard
proposes a model and notation for capturing single decision tables as well as
the interconnection of multiple decision tables and their relationship with
other forms of business knowledge. In addition, it proposes a clean integration
with business process models, with particular reference to BPMN, so as to
achieve a suitable separation of concerns between the process logic and the
decision logic \cite{Batoulis2015}.

For all these reasons, the standard has attracted the attention of both
academia and industry, giving a new momentum to the field of decision and rule
management and its interplay with business process management. The standard is
already receiving widespread adoption in the industry, and an increasing number
of tools and techniques are being developed to assist users in modeling,
verifying, and applying DMN models. This is, e.g., witnessed by the
incorporation of DMN inside the \emph{Signavio} toolchain for business process
management\footnote{\url{https://www.signavio.com/}}, and inside the open-source
\emph{OpenRules} business rules and decision management
system\footnote{\url{http://openrules.com/}}.

DMN builds on the notion of \emph{decision table}\footnote{The DMN standard
 uses the term \emph{decision}, but we prefer to use here \emph{decision table}
 to avoid ambiguity between the technical notion defined by DMN and the general
 notion of decision.}, defined as \emph{``the act of determining an output
 value (the chosen option), from a number of input values, using logic defining
 how the output is determined from the inputs''} \cite{DMN}. This is
diagrammatically related to the long-standing notion of \emph{decision table}
\cite{Pooch74,VanD93}, which consists of columns representing the inputs and
outputs of a decision, and rows denoting rules.  Concretely, DMN comes with two
languages for capturing the decision logic.
The most sophisticated language, called \feel (\emph{Friendly Enough Expression
 Language}), is a complex, textual specification language not apt to be used
and understood by domain experts, and that does not come with a graphical
notation. It is Turing-powerful since it relies on various mechanisms to
specify rules using complex arithmetic expressions and generic functions, in
turn expressed in \feel itself or in external languages such as Java.
The second decision logic specification language supported by \dmn, the one we
are actually considering in this paper, is called \sfeel
(\emph{Simplified-\feel}). \sfeel is equipped with a graphical notation that is
also defined in the standard. It is a simple rule-based language that employs
comparison operators between attributes and constants as atomic expressions,
which are then combined into more general conditions using boolean operators
(with a restricted usage of negation). Interestingly, \sfeel emerged as a
suitable trade-off between expressiveness and simplicity, and its main
principles come from previous, long-standing research in decision and rule
management (e.g., a very similar language is adopted in the well-established
\emph{Prologa} tool\footnote{\url{https://feb.kuleuven.be/prologa/}}).

While DMN decision tables are rooted in mainstream approaches to decision
management, a distinctive feature of the DMN standard itself is the combination
of multiple decision tables into more complex so-called \emph{decision
 requirements graphs} (\drgs), graphically depicted using \emph{decision
 requirements diagrams} (\drds). DRGs provide \emph{``a bridge between business
 process models and decision logic models"} \cite{DMN}: for every task in the
process model of interest where decision-making is required, a dedicated \drg
provides a separate definition of which decisions must be made within such a
task, together with the interrelationships of such decisions, and their
requirements for decision logic. In particular \drgs may rely on additional
so-called \emph{business knowledge models} for their functioning. Business
knowledge models are external to the \drg, and consequently the standard does
not dictate how they should be specified, nor how they semantically interact
with the internal decision logic expressed by the \drg.

In this work, we consider one of the most important types of business
knowledge, namely background knowledge that conceptually accounts for the
relevant, structural aspects of the domain of interest (such as entities and
their main relationships). As customary, we assume that this background
knowledge is explicitly encapsulated inside an ontology. The main issue that
arises when a \drg is integrated with an ontology is that the \drg should not
any longer be interpreted under the assumption of complete
information. Interestingly, this does not only affect the way the \drg is
applied on specific input data to compute corresponding outputs, but also
impacts on the intrinsic properties of the \drg, such as \emph{completeness}
(defined as the ability of \emph{``providing a decision result for any possible
 set of values of the input data"} \cite{DMN}). To tackle this fundamental
challenge, we introduce a combined framework, called \emph{semantic \dmn},
based on the notion of \emph{decision knowledge base} (\dkb). A
\dkb semantically combines a decision logic, modeled as a \dmn \drg whose
decision tables are expressed in \sfeel, with a general ontology formalized
using multi-sorted first-order logic.  The different sorts are used to
seamlessly integrate abstract domain objects with data values belonging to the
concrete domains used in the DMN rules (such as strings, integers, and reals).

We provide a logic-based semantics for \dkbs, thus proposing, to the best of
our knowledge, the first formalization of \drgs and of their integration with
background knowledge.  Due to the specific challenges posed by such an
integration, the formalization of the DMN decision tables contained in the \drg
is of independent interest, and represents a conceptual refinement of the
logic-based formalization proposed by \citeN{CDL16}.  We then approach the
problem of actually reasoning on \dkbs, on the one hand providing a
formalization of the most fundamental reasoning tasks, and on the other hand
giving insights on how they can be actually carried out, and with which
complexity. To this end, we need to restrict the expressive power of the
ontology language. In fact, we target the significant case where the ontology
consists of a description logic (DL) \cite{BCMNP07} knowledge base equipped
with \emph{datatypes} \cite{Lutz02d,SaCa12,ArKR12,W3Crec-OWL2-Overview}.  In
such a DL, besides the domain of abstract objects, one can refer to concrete
domains of data values (such as strings, integers, and reals) accessed through
functional relations. Complex conditions on such values can be formulated by
making use of \emph{unary} predicates over the concrete domains. The
restriction to unary predicates only, is what distinguishes DLs with datatypes
from the richer setting of DLs with concrete domains, where in general
arbitrary predicates over the datatype/concrete domain can be specified. We
demonstrate that these constructs are expressive enough to encode \drgs.

 Then,
we exploit this encoding to show that all the introduced reasoning tasks can be
decided in \exptime in the case where background knowledge is represented using
the DL \ALCHdd.  This DL is a strict sub-language of the ontology language
OWL\,2, which has been standardized by the W3C \cite{W3Crec-OWL2-Overview},
hence one can rely on standard OWL\,2 reasoners \cite{TsHo06,SiPa06,ShMH08} for
all reasoning tasks in \ALCHdd and on \dkbs.
However, this does not provide us with computationally optimal complexity
bounds.  On the one hand, while OWL\,2 and \ALCHdd are equipped with multiple
datatypes (which is also indicated by the letter $\types$ in the name of the
latter logic), reasoning in the very expressive DL $\mathit{SROIQ}(D)$, which
is the formal counterpart of OWL\,2, was initially studied for a single
datatype only (which is indicated by the letter $D$ in the name) \cite{HoKS06}.
As pointed out already by \citeN{HoKS06}, the proposed reasoning technique can
be extended to multiple datatypes by following the proposal by \citeN{PaHo03}.
Still, the adopted algorithms are tableaux-based, and while typically effective
in practical scenarios, they do not provide worst-case optimal computational
complexity bounds \cite{BaSa00}.  To show that reasoning in \ALCHdd can indeed
be carried out in worst-case single exponential time, we develop a novel
algorithm that is based on the \emph{knot} technique proposed by \citeN{OrSE08}
for query answering in expressive DLs.

To introduce the main motivations behind our proposal and show its
effectiveness, we consider a complex case study in maritime security extracted
from one of the challenges of the \emph{decision management
 community}\footnote{\url{https://dmcommunity.org/}}, arguing that our approach
facilitates modularity, separation of concerns, and understanding of how a
decision logic can be contextualized in a specific setting.

This article is an extended version of an article by
\citeN{CDMM17}. Differently from that work, we consider here the new version of
the standard (i.e., DMN~1.1), and we deal not only with single decision tables,
but also with their interconnection in a \drg. By considering \drgs, we extend
the logic-based formalization proposed by \citeN{CDMM17}, expand our case study
accordingly, and introduce new interesting properties that refer to the overall
decision logic encapsulated in a \drg.

The article is organized as follows.  In Section~\ref{sec:case} we present the
case study in maritime security.  In Section~\ref{sec:preliminaries} we
introduce the two formalism we use in our formalization of \dmn \drgs, namely
multi-sorted first-order logic, and DLs extended with datatypes, and we define
\dmn decision tables according to DMN~1.1.  In Section~\ref{sec:dkb} we
introduce and formalize \dkbs, and discuss the reasoning tasks over them.  In
Section~\ref{sec:translation} we address the problem of reasoning over \dkbs by
resorting to a translation in DLs.  In Section~\ref{sec:related} we discuss
related work, and in Section~\ref{sec:conclusions} we draw final conclusions.

\section{Case Study}
\label{sec:case}

Our case study is inspired by the international \emph{Ship and Port Facility
 Security
Code}\footnote{\url{https://dmcommunity.wordpress.com/challenge/challenge-march-2016/}},
used by port authorities to determine whether a (cargo) ship can enter a Dutch
port.  On the one hand, this requires to decide ship clearance, that is,
whether a ship approaching the port can enter or not. On the other hand, we
also consider the communication of where the refueling station is located for
the approaching ship. Both such interrelated decisions depend on a combination
of ship-related data, some focusing on the physical characteristics of the
ship, and others on contingent information such as the transported cargo and
certificates exhibited by the owner of the ship.

\subsection{Domain Description}
\label{sec:domain-description}

To describe how ship clearance is handled by a port authority, three main
knowledge models, reported in \figurename~\ref{fig:ex}, have to be suitably
integrated:
\begin{itemize}
\item background domain knowledge, describing the different types of cargo
  ships and their physical characteristics (see
  \figurename~\ref{tab:shiptypes});
\item the clearance process, describing when clearance has to be assessed, how
  the different clearance-related tasks can unfold over time, which data must
  be collected, and which possible ending states exist (see
  \figurename~\ref{fig:bpmn-ex});
\item the clearance decision, capturing the decision logic that relates all the
  important ship data with the determination of whether the ship can enter or
  not, and where its refuel station is located (see
  \figurename~\ref{fig:dmn-ex}).
\end{itemize}
It is important to notice that the different knowledge models are not
necessarily developed and maintained by the same responsible authority, nor
co-evolve in a synchronous way. For example, the background knowledge may be
obtained by combining the catalogues produced by the different ship vendors,
and updated when vendors change the physical characteristics of the types of
ship they produce. The process may vary from port to port, still ensuring that
its functioning behaves in accordance with national regulations. Finally, the
clearance decision may contain decisions with different authorities: the
decision of whether a ship can enter into a port or not may be in fact handled
at the national level, keeping it aligned with the evolution of laws and norms,
whereas the determination of the refuel area may vary from port to port, so as
to reflect its physical characteristics and internal operational rules.

In the following, we detail each of the three knowledge models mentioned above.

\medskip
\noindent
\textbf{Business process.~}
The process adopted by the port authority is shown in
\figurename~\ref{fig:bpmn-ex} using the BPMN notation.  An instance of this
process is created whenever an entrance request is received by the port
authority from an approaching ship. The process management system immediately
extracts the main data associated to the ship, namely its identification code,
as well as its type. Then, two branches are executed in parallel. The first
branch is about performing a physical inspection of the ship, in particular to
determine the amount of cargo residuals carried by the ship. The second branch
deals instead with the acquisition of the ship certificate of registry, in
particular to extract its expiration date (for simplicity, we do not handle
here the case where the ship does not own a valid certificate).

Once all these data are obtained, a business rule task is used to decide about
whether the ship can enter into the port or not and, if so, where the refuel
area for that ship is located.  If the resulting decision concludes that the
ship cannot enter, the process terminates by communicating the refusal to the
ship.  If instead the ship is allowed to enter, the dock is opened and the
process terminates by informing the ship about the refuel area.

\begin{figure}[t]
\begin{subfigure}[b]{0.7\textwidth}
  \includegraphics[width=\textwidth]{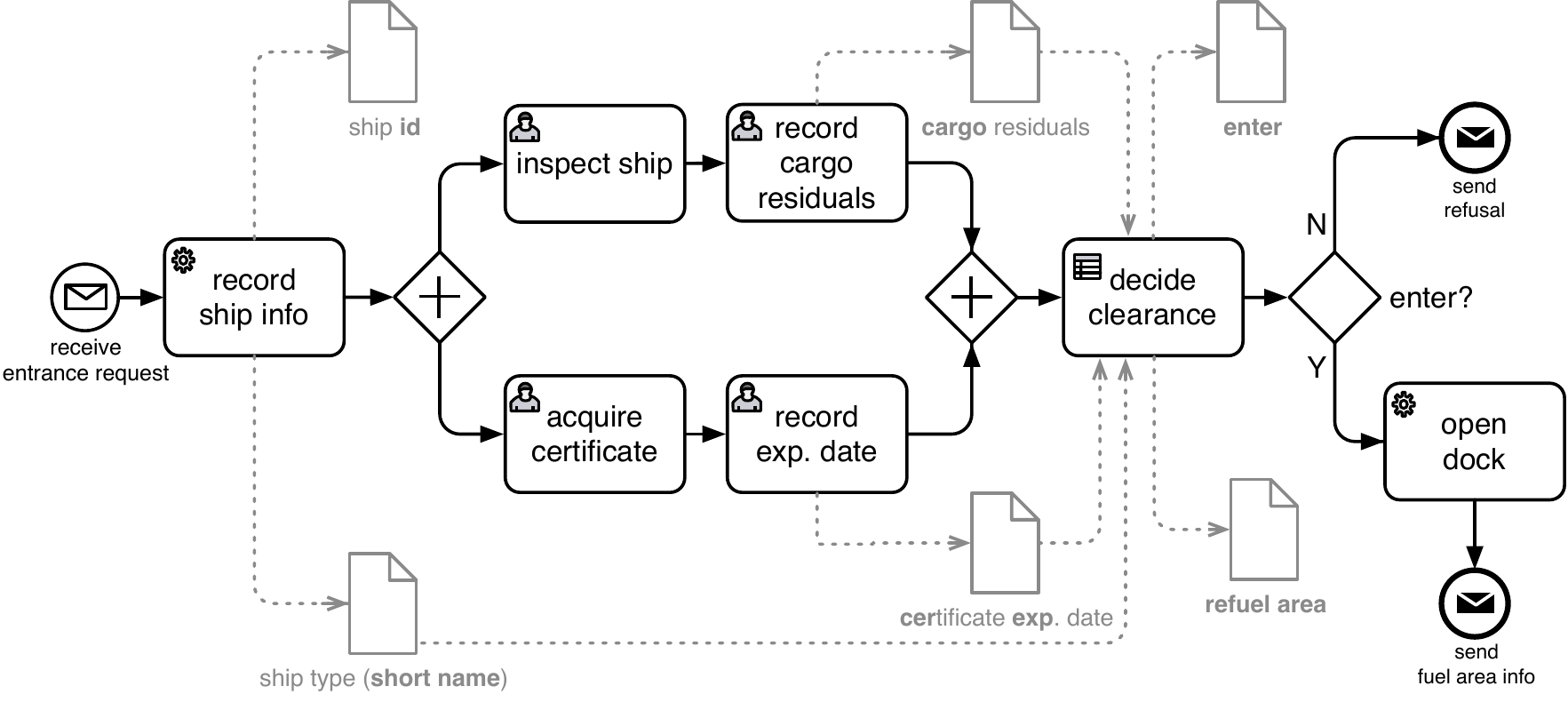}
  \caption{BPMN process}
  \label{fig:bpmn-ex}
\end{subfigure}
~
\begin{subfigure}[b]{0.28\textwidth}
  {\quad\raggedright\includegraphics[width=.85\textwidth]{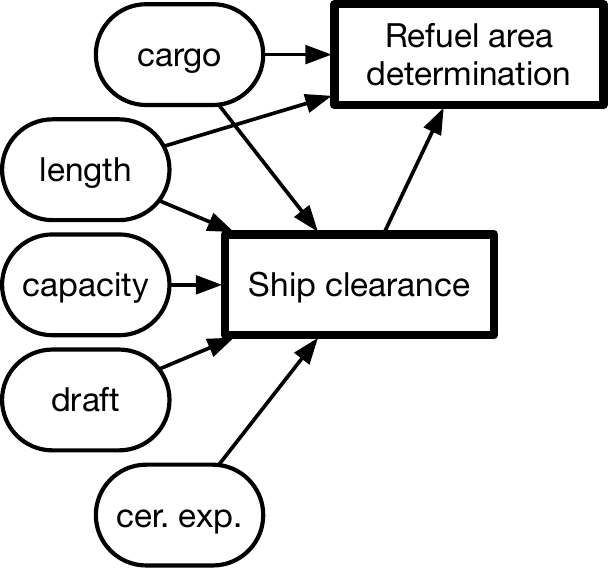}}
  \caption{\drd for \emph{decide clearance}}
  \label{fig:dmn-ex}
\end{subfigure}
\begin{subfigure}[b]{\textwidth}
\begin{tabular}
{
%|X|l||
%S[table-format=3] @{~} c @{~} l|
%S[table-format=2.1] @{~} c @{~} S[table-format=2]|
%S[table-format=5] @{~} c @{~} S[table-format=5]|
l@{\quad}l@{\qquad\qquad}l@{\quad}l@{\quad}l
}
\hline
\textbf{Ship Type} &
\textbf{Short Name} &
\textbf{Length} (m) &
\textbf{Draft} (m) &
\textbf{Capacity} (TEU) \\
\hline
\textit{Converted Cargo Vessel}
&
\rel{CCV}
& 135 & 0~--~9 & 500 \\
\textit{Converted Tanker}
&
\rel{CT}
& 200 & 0~--~9 & 800 \\
\textit{Cellular Containership}
&
\rel{CC}
& 215 & 10 & 1000~--~2500 \\
\textit{Small Panamax Class}
&
\rel{SPC}
& 250 & 11~--~12 & 3000 \\
\textit{Large Panamax Class}
&
\rel{LPC}
& 290 & 11~--~12 & 4000          \\
\textit{Post Panamax}
&
\rel{PP}
& 275~--~305 & 11~--~13 & 4000~--~5000   \\
\textit{Post Panamax Plus}
&
\rel{PPP}
& 335 & 13~--~14 & 5000~--~8000   \\
\textit{New Panamax}
&
\rel{NP}
& 397 & 15.5 & 11000~--~14500\\
\hline
\end{tabular}
\caption{Ontology of cargo ships and their physical characteristics}
\label{tab:shiptypes}
\end{subfigure}

\caption{Three knowledge models used by a port authority to determine clearance
 and location of the refuel area for a ship: (a) captures the business process
 using BPMN; (b) shows a \dmn \drd encapsulating the decision logic underlying
 the \emph{decide clearance} business rule task used in the process; the
 decision logic has to be understood in the context of the ship ontology
 depicted in (c), which provides the background knowledge to understand the
 relationship between ship types and their corresponding characteristics. The
 ship ontology is hence depicted as a business knowledge model in the \drd.}
\label{fig:ex}
\end{figure}

\medskip
\noindent
\textbf{Background domain knowledge.~}
In our setting, we consider background knowledge describing the different types
of ships that may enter the port, together with their physical characteristics:
\begin{compactitem}%[\itshape (i)]
\item \emph{length} of the ship (in $m$);
\item \emph{draft} size (in $m$);
\item \emph{capacity} of the ship (in $\mathit{TEU}$,
  which stands for Twenty-foot Equivalent Units).
\end{compactitem}
The taxonomy of ship types, together with the relationship between types and physical characteristics, is captured in a \emph{ship ontology}, depicted in
\tablename~\ref{tab:shiptypes} using an informal tabular format.

\medskip
\noindent
\textbf{Decision logic.~}
The decision logic consists of two decision tables: a \emph{ship clearance}
decision table used to determine whether a ship can enter a port or not, and a
\emph{refuel area determination} decision table used to compute which refuel
area should be used by a ship.  As pointed out before, we assume that the first
 table is fixed nationally, while the second is defined on a per-port
basis.

Let us first focus on ship clearance. A ship can enter the port only if the
ship complies with the requirements of the inspection.  This is the case if the
ship is equipped with a \emph{valid certificate of registry}, and the ship
meets the \emph{safety requirements}.  The certificate of registry owned by the
ship is considered valid if the certificate expiration date is after the
current date. The rules for establishing whether a ship meets the safety
requirements depend on the characteristics of the ship, and the amount of
residual cargo present in the ship. The limitation concerning residuals is specified in terms of concentration per space unit, fixing thresholds that depend on the capacity of the ship. This is because residual cargo has to be manually inspected by the port authority. Thresholds are then put to limit the amount of resources and time for the inspection. In addition, the concentration limits vary depending on the ship capacity so as to level the maximum amount of overall residuals to be inspected, irrespectively of the ship type.

In particular, small ships (with maximum
length 260\,m and maximum draft 10\,m) may enter provided that their capacity
does not exceed 1000\,TEU. Ships with a small length (maximum 260\,m), medium
draft comprised between 10\,m and 12\,m, and capacity not exceeding 4000\,TEU,
may enter only if the carried cargo residuals do not exceed 0.75\,mg dry weight
per cm$^{2}$. Ships of medium size (with length comprised between 260\,m and
320\,m excluded, and draft strictly bigger than 10\,m and not exceeding 13\,m),
and with a cargo capacity below 6000\,TEU, may enter only if their cargo
residuals do not exceed 0.5\,mg dry weight per cm$^{2}$.  Finally, big ships
with length comprised between 320\,m and 400\,m excluded, draft larger than
13\,m, and capacity exceeding 4000\,TEU, may enter only if their carried
residuals are at most 0.25\,mg dry weight per cm$^{2}$. Larger ships are not explicitly mentioned in the rules, and are therefore implicitly considered as not eligible for entering.

Let us now focus on the determination of the refuel area. This decision table
depends on ship clearance, and on some of the physical characteristics of the
ship, in particular length and draft size. On the one hand, if clearance is
rejected, then no area is assigned (this is represented using string \none). On
the other hand, the \indoor refuel area is preferred over the \outdoor area,
but it is not possible for too big ships to refuel indoor, due to physical
constraints. In particular, ships that are longer than 350\,m can only refuel
indoor if their cargo does not carry more than 3\,mg dry weight per cm$^{2}$.

\subsection{Challenges}
\label{sec:challenges}
The first challenge posed by this case study concerns modeling, representation,
management, and actual application of the clearance and refuel decision logic,
as well as its integration with business processes. All these issues are
tackled by the \dmn standard. In particular, the standard defines clear
guidelines on how to encode and graphically represent the input/output attributes and
the rules of interest in the form of decision tables, as well as to aggregate
them into a \drg that highlights how they interact with each other, and with other
business knowledge models.

Specifically, the \drg of our case study is graphically rendered using the \drd
of \figurename~\ref{fig:dmn-ex}. In this \drd:
\begin{itemize}
\item Rounded rectangles represent \emph{input data}, to be assigned externally
  when the decision logic has to be applied in a specific context.
\item Rectangles represent \emph{decision tables}, which may require as input
  either input data provided externally, or the output produced by other
  decision tables; decision tables producing output results that must be
  returned to the external world are represented with bold contour (in our case
  study, both decision tables are of this form).
\item Solid arrows represent \emph{information requirements}, indicating which
  input data or decision tables are used as input to other decision tables.
\item Rectangles with two clipped corners represent \emph{business knowledge
   models}, which may be used by decision tables to properly compute their
  input-output relation (see \figurename~\ref{fig:bkm-ex}).
\end{itemize}
Additional constructs are available to interconnect different \drgs and
describe authorities, i.e., sources of knowledge, but they are orthogonal to
the core aspects introduced above, so we do not consider such additional
constructs here.
%
% This table, in turn, may be used to document the decision logic for clearance
% determination, and to match the data of a ship with the modeled rules,
% computing the corresponding output(s), i.e., whether the ship can enter or
% not. this latter mechanism is backed up by a formal semantics in predicate
% logic \cite{CDL16}.

Each \dmn decision table can be decorated with meta-information capturing its
\emph{hit policy}, i.e., declare whether its contained rules are
non-overlapping or, if they do, how to reconcile the output values produced by
multiple rules that simultaneously trigger. In addition, both single decision
tables and complete \drgs shall be understood in terms of their
\emph{completeness}, i.e., whether they are able to produce a final output for
each possible configuration of the input data, or whether there are inputs for
which the decision table is undefined. This is particularly critical in the
case of \drgs, since incompleteness may be caused by internal mismatches
between decision tables interconnected via information requirements. The main
issue when it comes to such properties is that there is no guarantee that they
are actually reflected by the actual decision logic \cite{CDL16}. In addition,
while these are the main properties mentioned by the \dmn standard, there are
many more properties that should be checked so as to ascertain the correctness
of a \drg. For example, one could check whether all rules may potentially
trigger. In the case of a single decision table this problem boils down to
check whether certain rules are masked by others \cite{CDL16}. However, in the
more general case of a \drg, the fact that a rule never triggers may be related
again to the complex interconnection among multiple decision tables.

The main point of this work, though, is that the investigation of such
properties and, more in general, the meaning of a decision logic, cannot be
understood in isolation from background knowledge, but has instead to be
analyzed \emph{in the light} of such knowledge. Conceptually, this requires to
lift from an approach working under complete information to one that works
under \emph{incomplete information}, and where the background knowledge is used
to constrain, complement, and contextualize the decision logic. This interplay
is far from trivial, and impacts on the properties of a \drg and its contained
decision tables, their input-output semantics as well as, ultimately, their
correctness.

Here we discuss, using our case study, two of the most critical challenges when
it comes to understand \dmn in the presence of background knowledge. First and
foremost, let us consider in more detail the interplay between the BPMN process
in \figurename~\ref{fig:bpmn-ex}, and the \drg in
\figurename~\ref{fig:dmn-ex}. According to the standard, the integration
between a process and a \drg is realized by introducing a business rule task in
the process, then linking such a task to the \drg. This implicitly assumes a
clear information exchange between these two knowledge models. On the one hand,
when an instance of the business rule task is created in the context of a
specific process instance, the input data of the \drg are bound to actual
values obtained from the state process instance. On the other hand, the output
values produced by the \drg are made visible to the process instance, which may
rely on them to decide how to consequently route the instance.

In our case study, it is clear that, \emph{syntactically}, the process and the
\drg do not properly integrate with each other, in particular for what concerns
the input data of the \drg. On the one hand, the \drg applies a comprehensive
strategy, where all physical parameters of the ship are requested as input. On
the other hand, the process adopts a pragmatic approach, in which only the ship
type and the cargo residuals are recorded, without requiring the port personnel
to measure each single physical parameter of the ship. While it is clear that a
syntactic interconnection between the two knowledge models would not work, what
about a \emph{semantic interconnection} that considers the ship ontology as
background knowledge? It turns out, interestingly, that once the ship ontology
is inserted into the picture, the process and the \drg can properly
interoperate. In fact, once the type of a ship is acquired, the ontology allows
one to infer partial, but sufficient information about the physical
characteristics of the ship, so as to properly apply the \drg once also the
expiration date of the certificate, and the amount of cargo residuals, are
obtained. It is worth noting that the ship ontology could not be reduced to an
additional decision table component of the \drg: \tablename~\ref{tab:shiptypes}
is not a decision table, since it is not always possible to univocally compute
the ship characteristics from the type (see, e.g., the case of \emph{Post
 Panamax} ship type). In fact, the domain knowledge captured by
\tablename~\ref{tab:shiptypes} is a set of \emph{constraints}, implicitly
discriminating combinations of ship types and characteristics that are allowed
from those that are impossible.

A second, open challenge relates to how the formal properties of single tables
change when they are interconnected in a \drg, and/or interpreted in the
presence of background knowledge. Consider the ship clearance decision table
and its associated rules described above. By elaborating on such rules, one
would conclude that such rules are non-overlapping, and that they are
incomplete, since, e.g., they do not handle clearance of a long ship
($\geq$\,320\,m) with small draft ($\leq$\,10\,m). While the non-overlapping
property clearly holds also when the ship ontology is considered, this is not
the case for incompleteness. In fact, under the assumption that all possible
ship types are those listed in \tablename~\ref{tab:shiptypes}, one would infer
that all the allowed combinations of physical parameters as captured by the
ontology are actually covered by the ship clearance decision table, which is in
fact \emph{complete with respect to the ship ontology}. E.g., the table clearly
shows that the aforementioned combination of parameters is impossible: long
ships cannot have such a small draft.

Finally, consider the decision table for refuel area determination. It is easy
to see that the rules encapsulated in such a decision table are complete and
non-overlapping. However, once this decision table is interconnected to ship
clearance, it turns out that the \outdoor station is never selected. In fact,
such a station is selected for ships whose physical characteristics lead to
reject the entrance request and, in turn, to be assigned to \none refuel area
independently of the actual physical characteristics.

Identifying all such issues is extremely challenging, and this is why we
propose a framework that on the one hand formally defines the interplay between
the different knowledge models, and on the other hand provides automated
reasoning capabilities to actually check the overall properties of a \drg in
the presence of background knowledge, as well as compute the consequences of a
decision table when input data are only partially specified, if possible.

\section{Sources of Decision Knowledge}
\label{sec:preliminaries}

We now generalize the case study presented in Section~\ref{sec:case}, and
introduce the two main knowledge models of semantic \dmn: background knowledge
expressed using a logical theory enriched with datatypes, and decision logic
captured as a \dmn \drg.

%%% Local Variables:
%%% mode: latex
%%% TeX-master: "main"
%%% save-place: t
%%% End:

\subsection{Logics with Datatypes}
\label{sec:logics}

To capture background knowledge, we resort to a variant of multi-sorted
first-order logic (see, e.g., \citeNP{Ende01}), which we call \fold, where one 
sort $\univ$ denotes a domain of abstract objects, while the remaining sorts
represent a finite collection \types of datatypes.  We consider a countably
infinite set $\sig$ of predicates, where each $p\in\sig$ comes with an arity
$n$, and a signature $\Sig_p:\{1,\ldots,n\}\ra\types\uplus\{\univ\}$, mapping
each position of $p$ to one of the sorts.  \fold contains unary and binary
predicates only.  A unary predicate $N$ with $\Sig_N(1)=\univ$ is called a
\emph{concept}, a binary predicates $P$ with $\Sig_P(1)=\Sig_P(2)=\univ$
% is called
a \emph{role}, and a binary predicate $F$ with $\Sig_F(1)=\univ$ and
$\Sig_F(2)\in\types$
% is called
a \emph{feature}.

\begin{example}
  \label{ex:fol-axiom}
  The cargo ship ontology in \tablename~\ref{tab:shiptypes} should be
  interpreted as follows: each entry applies to a ship, and expresses how the
  specific ship type constrains the other features of the ship, namely length,
  draft, and capacity. Thus the first table entry is encoded in \fold as
  % \footnote{To simplify notation, we do not qualify quantifiers with a sort,
  % but we require that all occurrences of a variable are in positions to which
  % the respective predicate signatures associate the same sort.}
  \[
    \forall s. \rel{stype}(s,``CCV") \limp
    \begin{array}[t]{@{}l}
      \rel{Ship}(s) \land \forall \ell.(\rel{length}(s,\ell) \limp \ell=135)
      \land{}\\
      \forall d. (\rel{draft}(s,d) \limp d \geq 0 \land d \leq 9) \land
      \forall c. (\rel{capacity}(s,c) \limp c = 500)
    \end{array}
  \]
  where $\rel{Ship}$ is a concept, $\rel{stype}$ is a feature of sort string,
  while $\rel{length}$, $\rel{draft}$, and $\rel{capacity}$ are all features of
  sort real.
\end{example}

We consider also well-behaved fragments of \fold that are captured by
description logics (DLs) extended with datatypes.  For details on DLs, we refer
to \citeN{BCMNP07}, and for a survey of DLs equipped with datatypes (also
called, in fact, \emph{concrete domain}), to \citeN{Lutz02d}.  Here we adopt
the DL \ALCHdd, which is an extension of the well-known DL \ALCd \cite{Lutz02d}
in two orthogonal directions: on the one hand \ALCHdd allows one to express
inclusions between two roles and between two features, which is denoted by the
presence in the name of the logic of the letter $\mathsf{H}$, for role/features
hierarchies; on the other hand, \ALCHdd is equipped with multiple datatypes,
instead of a single one.  As for datatypes, we follow the proposal by
\citeN{MoHo08}, on which the OWL~2 datatype maps are based
\cite[Section~4]{W3Crec-OWL2-Syntax}, but we adopt some simplifications that
suffice for our purposes.

\smallskip
\noindent
\textbf{Datatypes.~}
%\begin{definition}
A \emph{(primitive) datatype} $\type$ is a pair
$\tup{\domain_\type,\sigp_\type}$, where $\domain_\type$ is the \emph{domain}
of values\footnote{We blur the distinction between \emph{value space} and
 \emph{lexical space} of OWL~2 datatypes, and consider the datatype domain
 elements as elements of the lexical space interpreted as themselves.}  of
$\type$, and $\sigp_\type$ is a (possibly infinite) set of \emph{facets},
denoting unary predicate symbols.  Each facet $S\in\sigp_\type$ comes with a
set $S^\type \subseteq \domain_\type$ that rigidly defines the semantics of $S$
as a subset of $\domain_\type$.  Given a primitive datatype $\type$,
\emph{datatypes $\dtype$ derived from $\type$} are defined according to the
following syntax
  \[
    \dtype ~\lora~ \type ~\mid~ \dtype_1\cup\dtype_2 ~\mid~
    \dtype_1\cap\dtype_2 ~\mid~ \dtype_1\setminus\dtype_2 ~\mid~
    \{\cv{v}_1,\ldots,\cv{v}_m\} ~\mid~ \type[S]
  \]
  where $S$ is a facet for $\type$, and $\cv{v}_1,\ldots,\cv{v}_m$ are
  datatype values in $\domain_\type$.  The domain of a derived datatype is
  obtained for $\cup$, $\cap$, and $\setminus$, by applying the corresponding
  set operator to the domains of the component datatypes, for
  $\{\cv{v}_1,\ldots,\cv{v}_m\}$ as the set
  $\{\cv{v}_1,\ldots,\cv{v}_m\}$, and for $\type[S]$ as $S^\type$.
%\end{definition}
%
In the remainder of the paper, we consider the (primitive) datatypes present in
the S-FEEL language of the DMN standard: strings equipped with equality, and
numerical datatypes, i.e., naturals, integers, rationals, and reals equipped
with their usual comparison operators (which, for simplicity, we all illustrate
using the same set of standard symbols $=$, $<$, $\leq$, $>$, $\geq$). We
denote this core set of datatypes as $\types$.  Other S-FEEL datatypes, such as
that of datetime, are syntactic sugar on top of $\types$.
% Notice that, in all such cases, datatype predicates are always binary.

A facet for one of these datatypes $\type\in\types$ is specified using a binary
comparison predicate $\odot$, together with a \emph{constraining value} $v$,
and is denoted as $\odot_v$.  E.g., using the facet $\leq_9$ of the primitive
datatype $\rel{real}$, we can define the derived datatype $\rel{real}[\leq_9]$,
whose value domain are the real numbers that are $\leq 9$.  In the following,
we abbreviate %
$\type[S_1]\cap\type[S_2]$ as $\type[S_1{\land} S_2]$, %
$\type[S_1]\cup\type[S_2]$ as $\type[S_1{\lor} S_2]$, and %
$\type[S_1]\setminus\type[S_2]$ as $\type[S_1{\land}\lnot S_2]$, where $S_1$
and $S_2$ are either facets or their combinations with Boolean operators.
% Following \cite{MKH11}, we introduce the notion of \emph{(datatype)
% restriction} $r$ over $\type$ as a pair $\tup{S,\cv{v}}$, where
% $S \in \sigp_\type$ and $\cv{v} \in \domain_\type$. Given
% $x \in \domain_\type$, we say that $x$ \emph{satisfies} restriction
% $r = \tup{S,\cv{v}}$, written $r(x)$, if $\tup{x,y} \in S^\type$.

\medskip

Let $\univ$ be a countably infinite universe of objects.  A (DL)
\emph{knowledge base with datatypes} (KB hereafter) is a tuple
% $\tup{\sig,\ft,T,A}$,
$\tup{\sig,T,A}$, where $\sig$ is the \emph{KB signature},
% $\ft$ is a \emph{feature typing function},
$T$ is the \emph{TBox} (capturing the intensional knowledge of the domain of
interest), and $A$ is the \emph{ABox} (capturing extensional knowledge). When
the focus is on the intensional knowledge only, we omit the ABox, and call the
% triple $\tup{\sig,\ft,T}$
pair $\tup{\sig,T}$ \emph{intensional KB} (IKB).
The form of $T$ and $A$ depends on the specific DL of interest.
% Here we consider variants of DLs of the \dllite family \cite{CDLLR07,ACKZ09}.
Next, we introduce each component for the DL \ALCHdd, which is equipped with
multiple datatypes.

\smallskip
\noindent
\textbf{Signature.~} In a DL with datatypes, the signature
$\sig = \csig \uplus \rsig \uplus \fsig$ of a KB is partitioned into three
disjoint sets:
\begin{inparaenum}[\itshape (i)]
\item a finite set $\csig$ of \emph{concept names}, which are unary predicates
  interpreted over $\univ$, each denoting a set of objects, called the
  \emph{instances} of the concept;
\item a finite set $\rsig$ of \emph{role names}, which are binary predicates
  connecting pairs of objects in $\univ$; and
\item a finite set $\fsig$ of \emph{features}, which are binary
  \emph{functional} predicates, connecting an object to at most one typed
  value.  In particular, each feature $F$ comes with its datatype
  $\type_F\in\types$, which constrains the values to which the feature can
  connect an object.  When a feature $F$ connects an object $o$ to a value
  $\cv{v}$ (of type $\type_F$), we say that $F$ is \emph{defined for} $o$ and
  that $\cv{v}$ is the \emph{$F$-value} of $o$.
\end{inparaenum}

% \smallskip
% \noindent
% \textbf{Feature typing function}.  $\ft: \fsig\ra\types$ is a total function
% that defines the feature types, by mapping each feature $F \in \fsig$ into a
% corresponding datatype $\ft(F)$.

\smallskip
\noindent
\textbf{Concepts and roles.}  Each DL  is characterized by a set of constructs
that allow one to obtain complex concept and role expressions, by starting
from concept and role names, and inductively applying such constructs.  The DL
\ALCHdd provides only concept constructs, and no constructs for roles or
features.  Hence, the only roles and features that might be used are atomic
ones, given simply by a role name $R\in\rsig$ or a feature name $F\in\fsig$,
respectively.  Instead, \emph{concepts} $C$ are defined according to the
following grammar, where $N\in\csig$ denotes a concept name:
\[
  C ~\lora~
  \top ~\mid~ \bot ~\mid~ N ~\mid~
  \NOT C ~\mid~ C_1\AND C_2 ~\mid~ C_1\OR C_2 ~\mid~
  \SOME{R}{C} ~\mid~ \ALL{R}{C} ~\mid~ \SOME{F}{\dtype} ~\mid~ \UNDEF{F}.
\]
The intuitive meaning of the concept constructs is as follows.
\begin{itemize}
\item $N$ denotes an atomic concept, given simply by a concept name in $\csig$.
\item $\top$ is called the \emph{top concept}, denoting the set of all objects
  in $\univ$.
\item $\bot$ is called the \emph{empty concept}, denoting the empty set.
\item $\NOT C$ is called the \emph{complement} of concept $C$, and it denotes
  the set of all objects in $\univ$ that are not instances of $C$.
\item $C_1 \AND C_2$ is the \emph{conjunction} and $C_1\OR C_2$ the
  \emph{disjunction} of concepts $C_1$ and $C_2$, respectively denoting
  intersection and union of the corresponding sets of instances;
\item $\SOME{R}{C}$ is called a \emph{qualified existential restriction}.
  Intuitively, it allows the modeler to single out those objects that are
  connected via (an instance of) role $R$ to some object that is an instance of
  concept $C$;
\item $\ALL{R}{C}$ is called a \emph{value restriction}.  Intuitively, it
  denotes the set of all those objects that are connected via role $R$ only to
  objects that are instances of concept $C$;
\item $\SOME{F}{\dtype}$, where $F\in\fsig$ is a feature, and $\dtype$ a
  datatype that is either $\type_F$ or a datatype derived from $\type_F$, is
  called a \emph{feature restriction}.  Intuitively, it denotes the set of
  those objects for which the $F$-value satisfies condition $\dtype$,
  interpreted in accordance with the underlying datatype;
\item $\UNDEF{F}$ denotes the set of those objects for which feature $F$ is not
  defined.
\end{itemize}
Notice that the above constructs are not all independent from each other.
Indeed:
\begin{compactitem}
\item $\top$ is equivalent to $\NOT\bot$;
\item by De Morgan's laws, we have that $C_1\OR C_2$ is equivalent to
  $\NOT(\NOT C_1\AND \NOT C_2)$;
\item qualified existential restriction and value restriction are dual
  constructs, since $\ALL{R}{C}$ is equivalent to $\NOT\SOME{R}{\NOT C}$;
\item $\UNDEF{F}$ is equivalent to $\NOT\SOME{F}{\type_F}$.
\end{compactitem}
We also observe that, since features are functional relations, we do not need a
counterpart of value restriction for features.  Indeed, we have that
$\NOT\SOME{F}{\dtype}$ is equivalent to
$\UNDEF{F}\OR\SOME{F}{(\type_F\setminus\dtype)}$.

\smallskip
\noindent
\textbf{TBox.~} $T$ is a finite set of universal FO axioms based on predicates
in $\sig$, and on predicates and values of datatypes in $\types$.
Specifically, an \ALCHdd TBox is a finite set of \emph{assertions} of the
following forms:
\[
  \begin{array}[b]{r@{~}c@{~}l@{\qquad}l}
    C_1 &\ISA& C_2 & \textit{(concept inclusion)},\\
    R_1 &\ISA& R_2 & \textit{(role inclusion)},\\
    F_1 &\ISA& F_2 & \textit{(feature inclusion)},
  \end{array}
  \qquad\qquad
  \begin{array}[b]{r@{~}c@{~}l@{\qquad}l}
    R_1 &\ISA& \NOT R_2 & \textit{(role disjointness)},\\
    F_1 &\ISA& \NOT F_2 & \textit{(feature disjointness)},
  \end{array}
\]
where $C_1$ and $C_2$ are two \ALCHdd concepts, $R_1$ and $R_2$ two roles, and
$F_1$ and $F_2$ two features.  Intuitively, the first type of inclusion
assertion models that whenever an object is an instance of $C_1$, then it is
also an instance of $C_2$, and similarly for the other two types of inclusion
assertions, considering respectively pairs of objects, and pairs consisting of
an object and a value.  Instead, disjointness assertions are used to model that
no pair that is an instance of a role/feature can also be an instance of
another role/feature.  Notice that there is no need for a separate concept
disjointness assertion, since it can be mimicked by using negation in the
concept appearing in the right-hand side of a concept inclusion.

\begin{example}
  \label{ex:ship-ontology}
  The \ALCHdd encoding of the first entry in \tablename~\ref{tab:shiptypes} is:
  \[
    \SOME{\rel{stype}}{\rel{string}}[=_{\cv{``CCV"}}]
    ~\ISA~
    \rel{Ship}
    \begin{array}[t]{@{}l}
      {}~\AND~ \ALL{\rel{length}}{\rel{real}[=_{135}]}\\
      {}~\AND~ \ALL{\rel{draft}}{\rel{real}[\geq_{0} \land \leq_{9}]} ~\AND~
      \ALL{\rel{capacity}}{\rel{real}[=_{500}]}
    \end{array}
  \]
  All other table entries can be formalized in a similar way. The entire table
  is then captured by the union of all so-obtained inclusion assertions, plus
  an assertion expressing that the types mentioned in
  \tablename~\ref{tab:shiptypes} exhaustively \emph{cover} all possible ship
  types:
  \[
    \rel{Ship}
    ~\ISA~
    \begin{array}{@{}l}
      \SOME{\rel{stype}}{\rel{string}[=_{\cv{``CCV"}}]} ~\OR~
      \SOME{\rel{stype}}{\rel{string}[=_{\cv{``CT"}}]} ~\OR~ \cdots ~\OR~
      \SOME{\rel{stype}}{\rel{string}[=_{\cv{``NP"}}]}
    \end{array}
    % \rel{CCV} ~\OR~ \rel{CT} ~\OR~ \rel{CC} ~\OR~ \rel{SPC} ~\OR~
    % ~\OR~ \rel{PP} ~\OR~ \rel{PPP} ~\OR~ \rel{NP}
  \]
\end{example}

\smallskip
\noindent
\textbf{ABox.~} The ABox $A$ is a finite set of \emph{assertions}, or
\emph{facts}, of the form $N(o)$, $P(o,o')$, or $F(o,\cv{v})$, where $N$ is a
concept name, $P$ a role name, $F$ a feature, $o,o'\in\univ$, and
$\cv{v} \in \domain_{\type_F}$.\footnote{For simplicity, we have assumed that
 the objects occurring in an ABox are elements of the domain $\univ$.  In other
 words, we have made the \emph{standard name assumption}.}
% We assume that the ABox is always consistent with the TBox it comes with.

\begin{figure}[tbp]
  \[
    \begin{array}{r@{~~}c@{~~}l}
      \Int{\top} &=& \univ\\
      \Int{\bot} &=& \emptyset\\
      \INT{\NOT C} &=& \univ \setminus \Int{C}\\
      \INT{C_1 \AND C_2} &=& \Int{C_1} \cap \Int{C_2}\\
      \INT{C_1 \OR C_2} &=& \Int{C_1} \cup \Int{C_2}\\
      \INT{\SOME{R}{C}} &=& \{x\in\univ \mid \exists y\in\univ
        \text{ such that } \tup{x,y}\in\Int{R} \text{ and } y\in\Int{C} \}\\
      \INT{\ALL{R}{C}} &=& \{x\in\univ \mid \forall y\in\univ, \text{ if }
        \tup{x,y} \in \Int{R} \text{ then } y \in \Int{C} \}\\
      \INT{\SOME{F}{\dtype}} &=&
        \{x\in\univ \mid \exists \cv{v}\in\domain_{\type_F} \text{ such that }
        \tup{x,\cv{v}}\in\Int{F} \text{ and } \cv{v}\in\dtype\}\\
      \INT{\UNDEF{F}} &=&
        \{x\in\univ \mid \lnot\exists \cv{v}\in\domain_{\type_F}
        \text{ such that } \tup{x,\cv{v}} \in \Int{F}\}
    \end{array}
  \]
  \caption{Semantics of the \ALCHdd concept constructs}
  \label{fig:dl-semantics-concepts}
\end{figure}

\begin{figure}[tbp]
  \[
    \begin{array}{c}
      \begin{array}[b]{r@{\quad\text{if}\quad}l}
        C_1 \ISA C_2 & \Int{C_1} \subseteq \Int{C_2};\\
        R_1 \ISA R_2 & \Int{R_1} \subseteq \Int{R_2};\\
        F_1 \ISA F_2 & \Int{F_1} \subseteq \Int{F_2};
      \end{array}
      \qquad\qquad\qquad
      \begin{array}[b]{r@{\quad\text{if}\quad}l}
        R_1 \ISA \NOT R_2 & \Int{R_1} \cap \Int{R_2}=\emptyset;\\
        F_1 \ISA \NOT F_2 & \Int{F_1} \cap \Int{F_2}=\emptyset;\\[2mm]
      \end{array}\\
      \begin{array}{r@{\quad\text{if}\quad}l}
        N(o) & o \in \Int{N};\\
        R(o,o') & \tup{o,o'} \in \Int{R};\\
        F(o,\cv{v}) & \tup{o,\cv{v}} \in \Int{F}.
      \end{array}
    \end{array}
  \]
  \caption{Satisfaction of \ALCHdd TBox and ABox assertions}
  \label{fig:dl-semantics-assertions}
\end{figure}

\smallskip
\noindent
\textbf{Semantics.~} The semantics of an \ALCHdd KB $K=\tup{\sig,T,A}$ relies,
as usual, on the notion of first-order interpretation $\I=\inter$ over the
domain $\dom\subseteq\univ$, where $\Int{\cdot}$ is an interpretation function
mapping each atomic concept $N$ in $T$ to a set $\Int{N}\subseteq\dom$, each
role $R$ to a binary relation $\Int{R} \subseteq \dom\times\dom$, and each
feature $F$ to a relation $\Int{F} \subseteq \dom\times\domain_{\type_F}$ that
is functional, i.e., such that, if
$\{\tup{d,\cv{v}_1}, \tup{d,\cv{v}_2}\}\subseteq \Int{F}$, then
$\cv{v}_1=\cv{v}_2$.  Complex concepts are interpreted as shown in
\figurename~\ref{fig:dl-semantics-concepts}, and when an interpretation $\I$
\emph{satisfies} a TBox assertion or an ABox assertion is shown in
\figurename~\ref{fig:dl-semantics-assertions}.
Finally, we say that $\I$ is a \emph{model} of $T$ if it satisfies all
inclusion assertions of $T$, and a \emph{model} of $K$ if it satisfies all
assertions of $T$ and $A$.

\smallskip
\noindent
\textbf{Reasoning in \ALCHdd.~}
We first recall the definition of the main reasoning tasks over DL KBs, which
we will use later to formalize reasoning over \dmn \drgs:
\begin{itemize}
\item \emph{TBox satisfiability}: given a TBox $T$, determine whether $T$
  admits a model.
\item \emph{Concept satisfiability with respect to a TBox}: given a TBox $T$
  and a concept $C$, determine whether $T$ admits a model $\I$ such that
  $\Int{C}\neq\emptyset$.
\item \emph{KB satisfiability}: given a KB $K$, determine whether $K$ admits a
  model.
\item \emph{Instance checking}:
  \begin{itemize}[--]
  \item for concepts: given a KB $K$, a concept $C$, and an object $o$,
    determine whether $o\in\Int{C}$, for every model $\I$ of $K$;
  \item for roles: given a KB $K$, a role $R$, and a pair of objects $o$, $o'$,
    determine whether $\tup{o,o'}\in\Int{R}$, for every model $\I$ of $K$;
  \item for features: given a KB $K$, a feature $F$, an object $o$, and a value
    $\cv{v}$, determine whether $\tup{o,\cv{v}}\in\Int{F}$, for every model
    $\I$ of $K$.
  \end{itemize}
\end{itemize}
TBox reasoning in $\ALC$ with a single concrete domain $\type$ is decidable in
\exptime (and hence is \exptime-complete) under the assumption that
\begin{inparaenum}[\itshape (i)]
\item the logic allows for unary concrete domain predicates only,
\item the concrete domain $\type$ is \emph{admissible} \cite{HaMW01,HoSa01},
  and
\item checking $\type$-satisfiability, i.e., the satisfiability of conjunctions
  of predicates of $\type$, is decidable in \exptime.
\end{inparaenum}
This follows from a slightly more general result shown by
\citeN[Section~2.4.1]{Lutz02c}.  Admissibility requires that the
set of predicate names is closed under negation and that it contains a
predicate name denoting the entire domain.
% and that it is decidable to check the satisfiability of conjunctions of
% predicates.
Hence, TBox reasoning in $\ALC$ extended with one of the concrete domains used
in DMN (e.g., integers or reals, with facets based on comparison predicates
together with a constraining value), is \exptime-complete.
The variant of DL with concrete domains that we consider here, \ALCHdd, makes
only use of unary concrete domain (i.e., datatype) predicates, but allows for
multiple datatypes, and also for role and feature inclusions.  Moreover, we are
also interested in reasoning in the presence of an ABox.  Hence, the above
decidability and complexity results do not directly apply.
%
% However, we can combine the technique by \citeN[Theorem~2.14]{Lutz02c} to
% encode away unary concrete domains, which however does not give an optimal
% computational complexity bound,
%
However, we can adapt to our needs a technique proposed by \citeN{OrSE08} and
refined by \citeN{ELOS09b} and \citeN{Orti10} for reasoning over a KB
(actually, to answer queries over a KB), to show the following result.

\begin{theorem}
  \label{thm:alcdd}
  Let $\types$ be a set of datatypes such that for all datatypes
  $\type\in\types$ checking $\type$-satisfiability is decidable in \exptime.
  Then, for \ALCHdd KBs, the problems of concept satisfiability with respect to
  a TBox, KB satisfiability, and instance checking are decidable in \exptime,
  and actually \exptime-complete.
\end{theorem}

\begin{proof}
It is well known that a concept $C$ is satisfiable with respect to a TBox $T$
iff the KB $\tup{\sig\cup\{N_n\},T\cup\{N_n\ISA C\},\{N_n(o_n)\}}$ is
satisfiable, where $N_n$ is a fresh concept not appearing in $T$, and $o_n$ is
a fresh object (see, e.g., \citeNP{BCMNP07}).
Also, an object $o$ is an instance of a concept $C$ with respect to a KB
$K=\tup{\sig,T,A}$ iff the KB
$\tup{\sig\cup\{N_n\},T\cup\{N_n\ISA \NOT C\},A\cup\{N_n(o)\}}$ is
unsatisfiable, where $N_n$ is a fresh concept name. (Similarly for role and
feature instance checking, exploiting the fact that in the TBox we can express
role and feature disjointness.)
Hence, both concept satisfiability and instance checking can be polynomially
reduced to KB satisfiability, and we need to consider only the latter problem.

In the rest of the proof we show how to check the satisfiability of an \ALCHdd
KB $K=\tup{\sig,T,A}$.  We make use of a variation of the \emph{mosaic}
technique commonly adopted in modal logics \cite{Neme86}, and which is based on
the search for small components of an interpretation that can be composed to
construct a model of a given KB.  Specifically, we borrow and adapt to our
needs the technique based on knots introduced for query answering in expressive
DLs by \citeN{OrSE08}, and later refined by \citeN{ELOS09b} and \citeN{Orti10}.

As a first step, for each object $o$ such that $F(o,\cv{v})\in A$, for some
$F\in\fsig$ and value $\cv{v}$, we modify $K$ as follows:
\begin{inparaenum}[\itshape (i)]
\item we add to $\sig$ a fresh concept name $N_o$;
\item we remove from $A$ the assertion $F(o,\cv{v})$, and replace it with the
  assertion $N_o(o)$;
\item we add to the TBox the concept inclusion
  $N_o\ISA \SOME{F}{\{\cv{v}\}}$.
\end{inparaenum}
Hence, in the following, we assume that the ABox contains only membership
assertions for concepts and roles (and not for features).  We also assume
w.l.o.g.\ that all concepts appearing in $T$ are in \emph{negation-normal form}
(NNF), i.e., negation has been pushed inside so as to appear only in front of
concept names and as difference inside datatypes.  Indeed, as is well known,
one can convert every concept of \ALC into an equivalent one in NNF by
exploiting De Morgan's laws and the duality between qualified existential
restriction and value restriction.  Moreover, as we have observed above,
$\NOT\SOME{F}{\dtype}$ is equivalent to
$\UNDEF{F}\OR\SOME{F}{(\type_F\setminus\dtype)}$, and $\NOT\UNDEF{F}$ is
equivalent to $\SOME{F}{\type_F}$.
Finally, we consider $\bot$ as an abbreviation for $A\AND \NOT A$, and $\top$
as an abbreviation for $A\OR \NOT A$, for some concept name $A\in\csig$.

We use $\cl{K}$ to denote the smallest set of concepts and objects that
contains every concept and every object in $K$ and that is closed under
sub-expressions and negation in NNF (denoted~$\NOTnnf$) applied to concepts.
Moreover, for each concrete domain $\type\in\types$, we define the set
$\Gamma_{\type}$ of \emph{$\type$-expressions used in $K$} as
\[
  \Gamma_{\type} = \{E \mid \SOME{F}{E} \text{ occurs in } T \text{ for some }
  F\in\fsig \text{ s.t.\ } \type_F=\type \}.
\]
Adapting a definition by \citeN{ELOS09b}, we define now suitable forms of
\emph{types}:
\begin{itemize}
\item A \emph{concept-type} for $K$ is a set $\tau\subseteq\cl{K}$ that
  contains at most one object and such that, for all concepts
  $C_1,C_2\in\cl{K}$:
  \begin{itemize}[--]
  \item if $C_1\in\tau$, then $\NOTnnf C_1\notin\tau$;
  \item if $C_1\AND C_2\in\tau$, then $\{C_1,C_2\}\subseteq\tau$;
  \item if $C_1\OR C_2\in\tau$, then $C_1\in\tau$ or $C_2\in\tau$;
  \item if $C_1\ISA C_2\in T$, then $\NOTnnf C_1\in\tau$ or $C_2\in\tau$;
  \item if $N(o)\in A$, then $o\notin\tau$ or $N\in\tau$.
  \end{itemize}
  % If $\tau$ contains an object $o$, we say that $\tau$ is a concept-type
  % \emph{for} $o$.
\item For each $\type\in\types$, a \emph{$\type$-type} is a set
  $\tau\subseteq\Gamma_{\type}$ such that $\bigwedge_{E\in\tau}E(x)$ is
  satisfiable in $\type$.
\item A \emph{role-type} for $K$ is a set $\rho\subseteq\rsig$ such that, for
  all $R_1,R_2\in\rsig$:
  \begin{itemize}[--]
  \item if $R_1\ISA R_2\in T$, then $R_1\notin\rho$ or $R_2\in\rho$;
  \item if $R_1\ISA \NOT R_2\in T$, then $R_1\notin\rho$ or $R_2\notin\rho$.
  \end{itemize}
\item A \emph{feature-type} for $K$ is a set $\rho\subseteq\fsig$ such that,
  for all $F_1,F_2\in\fsig$:
  \begin{itemize}[--]
  \item if $F_1\ISA F_2\in T$, then $F_1\notin\rho$ or $F_2\in\rho$;
  \item if $F_1\ISA \NOT F_2\in T$, then $F_1\notin\rho$ or $F_2\notin\rho$.
  \end{itemize}
\end{itemize}

We use the different forms of types to define \emph{knots} for $K$, each of
which can be viewed as a tree of depth $\leq 1$: the root represents an object
labeled with a subset of $\cl{K}$; each leaf represents either an object
labeled with a subset of $\cl{K}$, or a value of a datatype $\type$, labeled
with a satisfiable conjunction of datatype expression for $\type$; and each
edge is labeled either with a role-type or with a feature-type.
Formally, a \emph{knot} is a pair $\kappa=\tup{\tau,S}$ that consists of a
concept-type $\tau$ for $K$ (called \emph{root-type}), and a set $S$ with
$\card{S}\leq \card{\cl{K}}$.  The set $S$ consists of pairs
$\tup{\rho,\tau'}$, where either $\rho$ is a role-type and $\tau'$ a
concept-type for $K$, or $\rho$ is a feature-type and $\tau'$ a $\type$-type
(for some $\type\in\types$) for $K$.

We first define local consistency conditions for knots, ensuring that the knot
does not contain internal contradictions.  A knot $\kappa=\tup{\tau,S}$ is
\emph{$K$-consistent} if the following conditions hold:
\begin{itemize}
\item if $\SOME{R}{C}\in\tau$, then there is some $\tup{\rho,\tau'}\in S$ such
  that $R\in\rho$ and $C\in\tau'$;
\item if $\ALL{R}{C}\in\tau$, then for all $\tup{\rho,\tau'}\in S$ with
  $R\in\rho$, we have that $C\in\tau'$;
\item if $\SOME{F}{E}\in\tau$, then there is a unique $\tup{\rho,\tau'}\in S$
  such that $F\in\rho$, and moreover $E\in\tau'$;
\item if ${\UNDEF{F}}\in\tau$, then there is no $\tup{\rho,\tau'}\in S$ such
  that $F\in\rho$;
\item if $o\in\tau$ and $R(o,o')\in A$, then there is a unique
  $\tup{\rho,\tau'}\in S$ such that $o'\in\tau'$, and moreover $R\in\rho$.
\end{itemize}
A knot that is $K$-consistent respects the constraints that $T$ and $A$ impose
locally, but this does not ensure that the knot can be part of a model of $K$,
as there could be non-local constraints that cannot be satisfied in a model in
which the knot is present.  Therefore, we introduce a global condition that
ensures that a set of knots can be combined in a model of $K$.
Given a set $\Psi$ of knots, a knot $\tup{\tau,S}\in\Psi$ is
\emph{$\Psi$-consistent} if for each $\tup{\rho,\tau'}\in S$ there is a knot
$\tup{\tau',S'}\in\Psi$, for some $S'$.  The set $\Psi$ is \emph{$K$-coherent}
if
\begin{inparaenum}[\itshape (i)]
\item each knot in $\Psi$ is both $K$-consistent and $\Psi$-consistent, and
\item for each object $o$ appearing in $A$, there is exactly one knot
  $\tup{\tau,S}\in\Psi$ such that $o\in\tau$.
\end{inparaenum}

We show that $K$ is satisfiable iff there exists a $K$-coherent set of knots.
For the ``if'' direction, we construct a model $\I$ of $K$ from a $K$-coherent
set $\Psi$ of knots.  By item~\textit{(ii)} in the definition of $K$-coherence,
for each object $o$ appearing in $A$, $\Psi$ contains exactly one knot
$\kappa_o$ whose root-type satisfies the local conditions imposed by $K$ on
$o$.  We start by introducing such knots, and we repeatedly connect suitable
successor knots $\tup{\tau',S'}$ to the leaves of the trees that have
concept-type or $\type$-type (for a suitable $\type\in\types$) equal to
$\tau'$.  The existence of such successors is guaranteed by the fact that all
knots in $\Psi$ are $\Psi$-consistent.  Notice also that, since for an object
$o'$ the knot that has $o'$ in its concept-type is unique, in this way we will
introduce in the model exactly one knot (i.e., object) representing $o'$. It is
easy to verify that the resulting interpretation is indeed a model of $K$.
For the ``only-if'' direction, consider a model $\I=\inter$ of $K$, and define
the following mapping $\mu$ that assigns to each object $o\in\dom$ a knot
$\mu(o)=\tup{\tau_o,S_o}$, where:
\begin{itemize}
\item $\tau_o=\{C\in\cl{K} \mid o\in\Int{C}\}$, and
\item $S_o$ is obtained as follows:
  \begin{itemize}[--]
  \item for each object $o'\in\dom$ such that $\tup{o,o'}\in\Int{R}$, for some
    role $R\in\rsig$, the set $S_o$ contains $\tup{\rho_{o'},\tau_{o'}}$, where
    $\rho_{o'}=\{R\in\rsig \mid \tup{o,o'}\in\Int{R}\}$, and
    $\tau_{o'}=\{C\in\cl{K} \mid o'\in\Int{C}\}$;
  \item for each value $\cv{v}\in\dom_{\type}$, for some $\type\in\types$, such
    that $\tup{o,\cv{v}}\in\Int{F}$, for some feature $F\in\fsig$, the set
    $S_o$ contains $\tup{\rho_{\cv{v}},\tau_{\cv{v}}}$, where
    $\rho_{\cv{v}}=\{F\mid \tup{o,\cv{v}}\in\Int{F}\}$, and
    $\tau_{\cv{v}}=\{E\in\Gamma_{\type} \mid \cv{v}\in E\}$.
  \end{itemize}
\end{itemize}
It is straightforward to check that $\Psi=\{\mu(o)\mid o\in\dom\}$ is a
$K$-coherent set of knots.

It remains to show that the existence of a $K$-coherent set of knots can be
verified in time exponential in the size of $K$.  Let $c=\card{\cl{K}}$,
$r=\card{\rsig}$, and $f=\card{\fsig}$.
% let $m$ be the number of objects appearing in $K$
Notice that $\cl{K}$ contains a number of concepts that is linear in the size
of $K$.
Then the number of knots for $K$ is bounded by $2^c\cdot (2^r+2^f)\cdot 2^c$,
i.e., by an exponential in the size of $K$.  Moreover, each knot $\kappa$ is of
size polynomial in the size of $K$, and one can check in time polynomial in the
combined sizes of $\kappa$ and $K$ whether $\kappa$ is $K$-consistent.  The
number of $K$-coherent sets of knots is doubly exponential in the size of $K$.
However, the existence of a $K$-coherent set of knots can be checked in time
single exponential in the size of $K$ as follows.
First, we say that a knot $\tup{\tau,S'}$ is a \emph{reduct} of a knot
$\tup{\tau,S}$ if there are enumerations
$S=\{\tup{\rho_1,\tau_1},\ldots,\tup{\rho_\ell,\tau_\ell}\}$ and
$S'=\{\tup{\rho'_1,\tau'_1},\ldots,\tup{\rho'_h,\tau'_h}\}$ such that
\begin{inparaenum}[\itshape (i)]
\item $h\leq\ell$,
\item $\rho'_i\cup\tau'_i\subseteq\rho_i\cup\tau_i$ for all
  $i\in\{1,\ldots,h\}$, and
\item $h<\ell$, or $\rho'_i\cup\tau'_i\subset\rho_i\cup\tau_i$ for some
  $i\in\{1,\ldots,h\}$.
\end{inparaenum}
A knot $\kappa$ is \emph{$K$-min-consistent} if it is $K$-consistent and no
reduct of $\kappa$ is $K$-consistent.  Intuitively, each $K$-min-consistent
knot is a self-contained model building block for minimal models of $K$.
With this notion in place, we construct a $K$-coherent set $\Psi$, all of whose
knots are $K$-min-consistent.  To do so, we enumerate, for each object $o$
appearing in $K$, over the knots $\tup{\tau,S}$ that are $K$-min-consistent and
such that $o\in\tau$.  Specifically, for each $o$, we exhaustively consider for
$\tau$ all subsets of $\cl{K}$ containing $o$, and extend both $\tau$ and $S$
so as to satisfy the conditions of $K$-consistency.  $K$-min-consistency of the
obtained $\tup{\tau,S}$ is then checked by considering all reducts of
$\tup{\tau,S}$ and verifying that none is $K$-consistent.
If $K$ contains $n$ objects, there are at most $c^n$ sets consisting of $n$
knots that we have to consider in the above enumeration.  From each such set
$\Psi_{\mathit{obj}}$, we then try to construct a $K$-coherent set of knots as
follows: we first construct a set $\Psi_{\mathit{obj}}^K$ of knots by adding to
$\Psi_{\mathit{obj}}$ \emph{all} those knots $\tup{\tau,S}$ for which $\tau$
does not contain any object and that are $K$-min-consistent. (Such knots are
generated similarly to the ones in the above enumeration, except that we
exhaustively consider all subsets of $\cl{K}$ \emph{not} containing any
object.)  We then repeatedly remove from $\Psi_{\mathit{obj}}^K$ those knots
that are not $\Psi_{\mathit{obj}}^K$-consistent.  If we are not forced to
remove from $\Psi_{\mathit{obj}}^K$ any of the knots initially in
$\Psi_{\mathit{obj}}$ (i.e., whose $\tau$ contains an object), then the
resulting set of knots is $K$-coherent.  Instead, if we are forced to do so for
each set $\Psi_{\mathit{obj}}$ in the enumeration, then there is no
$K$-coherent set of knots.
Given that there are $c^n$ sets in the enumeration, and that for each such set
the check for the existence of a $K$-coherent set of knots requires to iterate
over exponentially many knots, the overall algorithm runs in time single
exponential in the size of $K$.
% Using arguments analogous to the ones provided by \citeN{OrSE08}, one can
% show that the existence of a $K$-coherent sets of knots can be checked in
% time single exponential in the size of $K$.
Together with the well-known \exptime lower-bound for reasoning in \ALC, this
shows the claim.
\end{proof}

\smallskip
\noindent
\textbf{Rich KBs.~}
We also consider rich KBs where axioms are specified in full \fold (and the
signature is that of a \fold theory). We call such KBs \emph{\fold KBs}.

\subsection{DMN Decision Table}
\label{sec:dmn}

To capture the business logic of a simple decision table, we rely on the
\dmn~1.1 standard, and in particular \dmn~1.1 decision tables expressed in the
\sfeel language.

As for single decision tables, we resort to the formal definitions introduced
by \citeN{CDL16} to capture the standard, but we update them so as to target
\dmn~1.1. We concentrate here on single-hit policies only, that is, policies
that define an interpretation of decision tables for which at most one rule
triggers and produces an output for an arbitrary configuration of the input
attributes.  This is because in the case of multiple-hit policies, multiple
output values may be collected at once in a list. However, \sfeel does not
provide list-handling constructs (which are instead covered by the full \feel),
and hence only single-hit policies combine well with \sfeel within a \drg.  As
for single-hit policies, we consider:
\begin{itemize}
\item \emph{unique hit policy} ($\upol$) -- indicating that rules do not
  overlap;
\item \emph{any hit policy} ($\apol$) -- indicating that whenever multiple
  overlapping rules simultaneously trigger, they compute exactly the same
  output values;
\item \emph{priority hit policy} ($\ppol$) -- indicating that whenever multiple
  overlapping rules simultaneously trigger, the matching rule with highest
  output priority is considered (details are given next).
\end{itemize}
We do not consider the \emph{first policy}, as it is considered bad practice in
the standard, and from the technical point of view it can be simulated using
the priority hit policy.

An \emph{S-FEEL DMN decision table} $\dt$ (called simply \emph{decision table}
in the following) is a tuple
$\tup{\tname,\tin,\tout,\atype,\infacet,\outrange,\outdef,\trules,\thit}$,
where:
\begin{itemize}
\item $\tname$ is the \emph{table name}.
\item $\tin$ and $\tout$ are disjoint, finite ordered sets of \emph{input} and
  \emph{output attributes}, respectively.
\item $\atype: \tin \uplus \tout \rightarrow \types$ is a \emph{typing
   function} that associates each input/output attribute to its corresponding
  datatype.\footnote{We use $\uplus$ to denote the \emph{disjoint union} between two sets.}
\item $\infacet$ is a \emph{facet function} that associates each input
  attribute $\attr{a} \in \tin$ to an S-FEEL condition over $\atype(\attr{a})$
  (see below), which identifies the allowed input values for $\attr{s}$.
\item $\outrange$ is an \emph{output range} function that associates each
  output attribute $\attr{b} \in \tout$ to an $n$-tuple over
  $\atype(\attr{b})^n$ of possible output values (equipped with an ordering).
\item $\outdef: \tout \rightarrow \types$ is a \emph{default assignment}
  (partial) function mapping some output attributes to corresponding default
  values.
\item $\trules$ is a finite set $\set{r_1,\ldots,r_p}$ of \emph{rules}. Each
  rule $r_k$ is a pair $\tup{\incond_k,\outval_k}$, where $\incond_k$ is an
  \emph{input entry function} that associates each input attribute
  $\attr{a} \in \tin$ to an S-FEEL condition over $\atype(\attr{a})$, and
  $\outval_k$ is an \emph{output entry function} that associates each output
  attribute $\attr{b} \in \tout$ to an object in $\atype(\attr{b})$.
%\item $\tcompl \in \set{\complete,\incomplete}$ is the
%  \emph{completeness indicator} - $\complete$ (resp., $\incomplete$) stands for
%  \emph{(in)complete} table.
  \item $\thit \in \set{\upol,\apol,\ppol}$ % ,\opol,\rpol,\cpol}$
is the \emph{(single) hit policy indicator} for the decision table.
\end{itemize}
Notice that the ordering induced by the attributes in $\tout$, followed,
attribute by attribute, by the ordering of values in $\outrange$, is the one
upon which the priority hit indicator is defined, where the ordering is
interpreted by decreasing priority. Notice that rules with exactly the same
output values have the same priority, but this is harmless since they produce
the same result. To simplify the treatment, we introduce a total ordering
$\higherp$ over rules that respects the partial ordering induced by the output
priority, and that fixes an (arbitrary) ordering over equal-priority rules.

In the following, we use a dot notation to single out an element of a decision
table.  For example, $\get{\dt}{\tin}$ denotes the set of input attributes for
decision table $\dt$.

An \emph{(\sfeel) condition} $\cond$ over type $\type$ is inductively defined
as follows:
\begin{itemize}
\item ``$\anycond$" is the \emph{any value} condition (i.e., it matches every
  object in $\domain_\type$);
\item given a constant $\cv{v}$, expressions ``$\cv{v}$" and
  ``$\mathtt{not(}\cv{v}\mathtt{)}$" are \sfeel conditions respectively
  denoting that the value shall and shall not match with $\cv{v}$;
\item if $\type$ is numerical, given two numbers
  $\cv{v}_1, \cv{v}_2 \in \domain_\type$, the interval expressions
  ``$[\cv{v}_1, \cv{v}_2]$", ``$[\cv{v}_1, \cv{v}_2)$",
  ``$(\cv{v}_1, \cv{v}_2]$", and ``$(\cv{v}_1, \cv{v}_2)$" are S-FEEL
  conditions (interpreted in the standard way as closed, open, and half-open intervals);
\item given two S-FEEL conditions $\cond_1$ and $\cond_2$,
  ``$\cond_1,\cond_2$'' is a disjunctive S-FEEL condition that evaluates to
  true for a value $\cv{v} \in \domain_\type$ if either $\cond_1$ or $\cond_2$
  evaluates to true for $\cv{v}$.
\end{itemize}
% Notice that, for numerical domains, ``$\mathtt{not(}\cv{v}\mathtt{)}$" is
% just syntactic %sugar, as it can be re-expressed

\begin{example}
\label{ex:dmn}
\newcommand{\curr}{\ensuremath{\mathtt{today}}}
\newcommand{\yesenter}{\ensuremath{\mathtt{y}}}
\newcommand{\noenter}{\ensuremath{\mathtt{n}}}

\newcommand{\dmntext}[1]{\phantom{$\mathtt{[d}$}#1\phantom{\yesenter$]$}}

\tikzset{
table/.style={
  matrix of nodes,
  row sep=-\pgflinewidth,
  column sep=-\pgflinewidth,
  nodes={
    rectangle,
    draw=black,
    minimum width=1cm,
    minimum height=5mm,
    align=center },
  text depth=0.25ex,
  text height=1ex,
  nodes in empty cells
  },
  dmn/.style={
    matrix of nodes,
    row sep=-\pgflinewidth,
    column sep=-\pgflinewidth,
    nodes={
      rectangle,
      draw=black,
      text width=19mm,
      minimum height=5mm,
      align=center },
    nodes in empty cells,
  },
  dmnhit/.style={
    rectangle,
    draw,
    minimum height=15.2mm,
    minimum width=7.1mm,
    xshift=1.3mm,
    yshift=-\pgflinewidth
  },
  dmnrulen/.style={
    matrix of nodes,
    row sep=-\pgflinewidth,
    column sep=-\pgflinewidth,
    nodes={
      rectangle,
      draw=black,
      text width=5mm,
      minimum height=5mm,
      align=center },
    nodes in empty cells,
  },
}

\begin{table}[t]
\begin{tikzpicture}[node distance =0pt and 0.5cm]

\matrix[dmn]
  (in) {
  |[fill=incolor,minimum height=1cm]|
  Cer.~Exp.\qquad\qquad\qquad (date)
  &
  |[fill=incolor,minimum height=1cm]|
  Length \qquad\qquad\qquad (m)
  &
  |[fill=incolor,minimum height=1cm]|
  Draft \qquad\qquad\qquad (m)
  &
  |[fill=incolor,minimum height=1cm]|
  Capacity \qquad\qquad\qquad (TEU)
  &
  |[fill=incolor,minimum height=1cm]|
  Cargo \qquad\qquad\qquad (mg/cm$^2$)
  \\
  \dmntext{$\geq 0$} &
  \dmntext{$\geq 0$} &
  \dmntext{$\geq 0$} &
  \dmntext{$\geq 0$} &
  \dmntext{$\geq 0$} \\
};

\matrix[dmn,anchor=north west,xshift=-2mm,yshift=-.1mm]
  (out)
  at (in.north east) {
  |[fill=outcolor,minimum height=1cm]|
  Enter\qquad\qquad\qquad \phantom{.}
  \\
  \dmntext{\yesenter, \noenter}\\
};

\node[dmnhit,
      anchor=east,
      yshift=\pgflinewidth]
  (hit)
  at (in.west)
  {\textbf{U}};

\node[rectangle,
      draw,
      anchor=south west,
      minimum height=8mm,
      minimum width=35mm,
      yshift=-\pgflinewidth]
  (title)
  at (hit.north west)
  {\textbf{Ship Clearance}};

\matrix[dmn,anchor=north west,yshift=1.9mm]
  (rules)
  at (in.south west) {
    \dmntext{$\leq \curr$}
    &
    \dmntext{$\anycond$}
    &
    \dmntext{$\anycond$}
    &
    \dmntext{$\anycond$}
    &
    \dmntext{$\anycond$}
    \\
    \dmntext{$> \curr$}
    &
    \dmntext{$< 260$}
    &
    \dmntext{$< 10$}
    &
    \dmntext{$< 1000$}
    &
    \dmntext{$\anycond$}
    \\
    \dmntext{$> \curr$}
    &
    \dmntext{$< 260$}
    &
    \dmntext{$< 10$}
    &
    \dmntext{$\geq 1000$}
    &
    \dmntext{$\anycond$}
    \\
    \dmntext{$> \curr$}
    &
    \dmntext{$< 260$}
    &
    \dmntext{$[10,12]$}
    &
    \dmntext{$< 4000$}
    &
    \dmntext{$\leq 0.75$}
    \\
    \dmntext{$> \curr$}
    &
    \dmntext{$< 260$}
    &
    \dmntext{$[10,12]$}
    &
    \dmntext{$< 4000$}
    &
    \dmntext{$>0.75$}
    \\
    \dmntext{$> \curr$}
    &
    \dmntext{$[260,320)$}
    &
    \dmntext{$(10,13]$}
    &
    \dmntext{$< 6000$}
    &
    \dmntext{$\leq 0.5$}
    \\
    \dmntext{$> \curr$}
    &
    \dmntext{$[260,320)$}
    &
    \dmntext{$(10,13]$}
    &
    \dmntext{$< 6000$}
    &
    \dmntext{$> 0.5$}
    \\
    \dmntext{$> \curr$}
    &
    \dmntext{$[320,400)$}
    &
    \dmntext{$\geq 13$}
    &
    \dmntext{$> 4000$}
    &
    \dmntext{$\leq 0.25$}
    \\
    \dmntext{$> \curr$}
    &
    \dmntext{$[320,400)$}
    &
    \dmntext{$\geq 13$}
    &
    \dmntext{$> 4000$}
    &
    \dmntext{$> 0.25$}
    \\
};

\matrix[dmn,
        anchor=north west,
        yshift=1.9mm]
  (ruleconclusions)
  at (out.south west) {
  \dmntext{\noenter} \\
  \dmntext{\yesenter} \\
  \dmntext{\noenter} \\
  \dmntext{\yesenter} \\
  \dmntext{\noenter} \\
  \dmntext{\yesenter} \\
  \dmntext{\noenter} \\
  \dmntext{\yesenter} \\
  \dmntext{\noenter} \\
};

\matrix[dmnrulen,
        anchor=north east,
        xshift=2.4mm]
  (rulenum)
  at (rules.north west) {
    \dmntext{1}\\
    \dmntext{2}\\
    \dmntext{3}\\
    \dmntext{4}\\
    \dmntext{5}\\
    \dmntext{6}\\
    \dmntext{7}\\
    \dmntext{8}\\
    \dmntext{9}\\
};
\end{tikzpicture}
\caption{\dmn representation of the \emph{ship clearance} decision of
 Figure~\ref{fig:dmn-ex}}
\label{tab:dmn-clearance}
\end{table}

%%%%%%%%%%%%%%%%%%%%%%%%%%%%%%%%%

\begin{table}
\begin{tikzpicture}[node distance =0pt and 0.5cm]

\matrix[dmn]
  (in) {
  |[fill=incolor,minimum height=1cm]|
  Enter\qquad\qquad\qquad \phantom{(m)}
  &
  |[fill=incolor,minimum height=1cm]|
  Length \qquad\qquad\qquad (m)
  &
  |[fill=incolor,minimum height=1cm]|
  Cargo \qquad\qquad\qquad (mg/cm$^2$)
  \\
  \dmntext{\yesenter,\noenter} &
  \dmntext{$\geq 0$} &
  \dmntext{$\geq 0$} \\
};

\matrix[dmn,anchor=north west,xshift=-2mm,yshift=-\pgflinewidth]
  (out) at (in.north east) {
  |[fill=outcolor,minimum height=1cm,minimum width=4cm]|
  Refuel Area\qquad\qquad\qquad \phantom{.}
  \\
  |[minimum width=4cm]|
  \hspace*{-8mm}\phantom{$\mathtt{[]}$}\underline{\none},~\indoor,~\outdoor\\
};

\node[dmnhit,anchor=east,
      yshift=\pgflinewidth]
  (hit)
  at (in.west)
  {\textbf{U}};

\node[rectangle,
      draw,
      anchor=south west,
      minimum height=8mm,
      minimum width=4cm,
      yshift=-\pgflinewidth]
  (title)
  at (hit.north west)
  {\textbf{Refuel Area Determination}};

\matrix[dmn,anchor=north west,yshift=1.9mm]
  (rules)
  at (in.south west) {
    \dmntext{\noenter}
    &
    \dmntext{$\anycond$}
    &
    \dmntext{$\anycond$}
    \\
    \dmntext{\yesenter}
    &
    \dmntext{$\leq 350$}
    &
    \dmntext{$\anycond$}
    \\
    \dmntext{\yesenter}
    &
   \dmntext{$> 350$}
    &
    \dmntext{$\leq 0.3$}
    \\
    \dmntext{\yesenter}
    &
    \dmntext{$> 350$}
    &
    \dmntext{$>0.3$}
    \\
};

\matrix[dmn,
        anchor=north west,
        yshift=1.9mm]
  (ruleconclusions)
  at (out.south west) {
  |[minimum width=4cm]|\dmntext{\none} \\
  |[minimum width=4cm]|\dmntext{\indoor} \\
  |[minimum width=4cm]|\dmntext{\indoor} \\
  |[minimum width=4cm]|\dmntext{\outdoor} \\
};

\matrix[dmnrulen,
        anchor=north east,
        xshift=2.4mm]
  (rulenum)
  at (rules.north west) {
    \dmntext{1}\\
    \dmntext{2}\\
    \dmntext{3}\\
    \dmntext{4}\\
};
\end{tikzpicture}
\caption{\dmn representation of the \emph{refuel area determination} decision
 of Figure~\ref{fig:dmn-ex}}
\label{tab:dmn-refuel}
\end{table}

%%% Local Variables:
%%% mode: latex
%%% TeX-master: "main"
%%% save-place: t
%%% End:

Tables~\ref{tab:dmn-clearance} and~\ref{tab:dmn-refuel} respectively show the
\dmn encoding of the \emph{ship clearance} and \emph{refuel area determination}
decision tables of our case study (cf.~Section~\ref{sec:case}).  The tabular
representation of decision tables obeys to the following standard conventions.
The first two rows (below the table title) indicate the table meta-information.
In particular, the leftmost cell reports the hit indicator, which, in both
tables, corresponds to unique hit.  Blue-colored cells (i.e., all other cells
but the rightmost one), together with the cells below, respectively model the
input attributes of the decision table, and which values they may assume.  This
latter aspect is captured by facets over their corresponding datatypes.  In
Table~\ref{tab:dmn-clearance}, the input attributes are:
\begin{inparaenum}[\itshape (i)]
\item the certificate expiration date,
\item the length,
\item the size,
\item the capacity, and
\item the amount of cargo residuals
\end{inparaenum}
of a ship.  Such attributes are nonnegative real numbers; this is captured by
typing them as reals, adding restriction ``$\geq 0$'' as facet.  The rightmost,
red cell represents the output attribute. In both cases, there is only one
output attribute, of type string.  The cell below enumerates the possible
output values produced by the decision table, in descending priority order.  If
a default output is defined, it is underlined.  This is the case for the \none
string in Table~\ref{tab:dmn-refuel}.

Every other row models a rule.  The intuitive interpretation of such rules
relies on the usual ``if \ldots then \ldots'' pattern. For example, the first
rule of Table~\ref{tab:dmn-clearance} states that, if the certificate of the
ship is expired, then the ship cannot enter the port, that is, the \emph{enter}
output attribute is set to $\noenter$ (regardless of the other input
attributes).  The second rule, instead, states that, if the ship has a valid
certificate, a length shorter than 260\,m, a draft smaller than 10\,m, and a
capacity smaller than 1000\,TEU, then the ship is allowed to enter the port
(regardless of the cargo residuals it carries). Other rules are interpreted
similarly.
\end{example}

\subsection{Decision Requirements Graphs}
\label{sec:dmn-drg}

We now formally define the notion of \drg in accordance with \dmn~1.1. As
pointed out before, we do not consider the contribution of authorities, but we
accommodate business knowledge models. Since they are considered external
elements to \drgs, in this phase they are simply introduced without a further
definition.  We will come back to this in Section~\ref{sec:dkb}.

A \emph{\drg} is a tuple $\tup{\tin,\infacet,\dts,\odts,\km,{\oimap},{\ireq}}$,
where:
\begin{itemize}
\item $\tin$ is a set of \emph{input data}, and $\infacet$ is a \emph{facet
   function} defined over $\tin$.
\item $\dts$ is a set of \emph{decision tables}, as defined in
  Section~\ref{sec:dmn}, and $\odts \subseteq \dts$ are the \emph{output
   decision tables}.  We assume that each decision table in $\dts$ has a
  distinct name that can be used to unambiguously refer to it within the \drg.
\item $\km$ is a set of \emph{business knowledge models}.
\item $\oimap: (\tin \cup \bigcup_{\dt \in \dts} \get{\dt}{\tout}) \times
  \bigcup_{\dt \in \dts} \get{\dt}{\tin} $ is an \emph{information flow}, that
  is, an \emph{output-unambiguous relation} connecting input data and output
  attributes of the decision tables in $\dts$ to input attributes of decision
  tables in $\dts$, where output-unambiguity is defined as follows:
  \begin{itemize}[$\bullet$]
  \item for every input attribute
    $\attr{a} \in \bigcup_{\dt \in \dts} \get{\dt}{\tin}$, there is at most one
    element $e$ such that $e \oimap \attr{a}$.
  \end{itemize}
\item $\ireq \subseteq \tin \times \dts \cup \dts \times \dts \cup \km \times
  \km \cup \km \times \dts$ is a set of \emph{information requirements}, relating 
  knowledge models to decision tables, knowledge models to other knowledge models,
  input attributes of the \drg to decision tables, and decision tables
  to other decision tables. Information requirements must guarantee \emph{compatibility} with $\oimap$,
  defined as follows:
  \begin{itemize}[$\bullet$]
  \item for every $\attr{i} \in \tin$ and every $\dt \in \dts$, we have
    $\attr{i} \ireq \dt$ if and only if there exists an attribute
    $\attr{b} \in \get{\dt}{\tin}$ such that $\attr{i} \oimap \attr{b}$;
  \item for every $\dt_o,\dt_i \in \dts$, we have $\dt_o \ireq \dt_i$ if and
    only if there exist attribute $\attr{b} \in \get{\dt_o}{\tout}$ and
    attribute $\attr{a} \in \get{\dt_i}{\tin}$ such that
    $\attr{b} \oimap \attr{a}$.
  \end{itemize}
\end{itemize}
In accordance with the standard, the directed graph induced by $\ireq$ over the
decision tables in $\dts$ must be \emph{acyclic}.  This ensures that there are
well-defined dependencies among decision tables.  While the standard introduces
information requirement variables to capture the data flow across decision
tables, here we opt for the simpler mathematical formalization based on the
information flow relation.

Also for \drgs, we employ a dot notation to single out their constitutive
elements (when clear from the context, though, we simply use $\oimap$ and
$\ireq$ directly). Given a \drg $\adrg$, we identify the set of \emph{free
 inputs} of $\adrg$, written $\freein{\adrg}$, as the set of input data of
$\adrg$ together with the input attributes of tables in $\adrg$ that are not
pointed by the information flow of $\adrg$:
\[
  \freein{\adrg} = \get{\adrg}{\tin} \cup \set{\attr{a} \mid \attr{a} \in
   \get{\dt}{\tin} \text{ for some } \dt \in \get{\adrg}{\dts},
   \text{ and there is no } x \text{ s.t.~} x \oimap \attr{a}}
\]
where $\oimap$ is the information flow of $\adrg$.  Complementarily, we call
\emph{bound attributes} of $\adrg$, written $\battr{\adrg}$, all the attributes
appearing in $\adrg$ that do not belong to $\freein{\adrg}$.  Such attributes
are all the output attributes used inside the decision tables of $\adrg$, but
also the input attributes that are bound to an incoming information flow.

Finally, we say that $e_1 \oimaptrans e_n$ if there exists a sequence
$\tup{e_2,\ldots,e_{n-1}}$ such that for each $i \in \set{1,\ldots,n-1}$, we
have $e_i \oimap e_{i+1}$.

\begin{example}
  The \drg of our case study, shown in Figure~\ref{fig:dmn-ex}, interconnects
  the input data and the two decision tables of \emph{ship clearance} and
  \emph{refuel area determination}, by setting as information flow the one that
  simply maps input data and output attributes to input attributes sharing the
  same name.  For example, the $\attr{Enter}$ output attribute of \emph{ship
   clearance} is mapped to the $\attr{Enter}$ input attribute of \emph{refuel
   area determination}. %Such two attributes need to be assigned to the same  name by any global naming defined over the \drg.
\end{example}

%%% Local Variables:
%%% mode: latex
%%% TeX-master: "main"
%%% save-place: t
%%% End:

\section{Semantic Decision Models}
\label{sec:dkb}

We now substantiate the integration between decision logic and background
knowledge, by introducing the notion of \emph{decision knowledge base} (\dkb),
which combines DMN \drgs with \fold knowledge bases, so as to \emph{empower DMN
 with semantics}.

\subsection{Decision Knowledge Bases}

The intuition behind our proposal for integration is to consider the decision
logic as a sort of enrichment of a KB describing the structural aspects of a
domain of interest. In this respect, the \drg is linked to a specific concept
of the KB. The idea is that given an object $o$ that is member of that concept,
$\dt$ inspects the feature of $o$ that matches the input data of the \drg,
triggering the corresponding decision logic. Depending on which rule(s) match,
$\dt$ then dictates which are the values to which $o$ must be connected via
those features that correspond to the output attributes
$\get{\dt}{\tout}$. Hence, the KB and the \drg interact on (some of) the free
inputs of the overall, complex decision, while the output attributes and the
bound inputs are exclusively present in the \drg, and are in fact used to infer
new knowledge about the domain. Since we work under incomplete information, we
also accept \drgs in which not all input attributes are fed by input data or by
the output of other decisions.

Formally, a \emph{decision knowledge base} over datatypes $\types$
($\types$-\dkb, or \dkb for short) is a tuple
$\tup{\sig,\TBox,\adrg,\cbridge,\ABox}$, where:
        \begin{itemize}
        \item $\TBox$ is a \fold IKB with signature $\sig$.
        \item $\adrg$ is a decision table that satisfies the following two typing conditions:
        \begin{compactdesc}
        \item[\textit{\mdseries(free input type compatibility)}] for every binary predicate $P \in \sig$ whose name appears in $\freein{\adrg}$, their datatypes coincide.        \item[\textit{\mdseries(uniqueness of bound attributes)}] For every bound attribute $\attr{a}$ in $\battr{\adrg}$, we have that no predicate $P \in \sig$ corresponds to $\attr{a}$.
        \end{compactdesc}
        \item $\cbridge \in \csig$ is a \emph{bridge concept}, that is, a concept from $\sig$ that links $\TBox$ with $\adrg$.
        \item $\ABox$ is an ABox over the extended signature $\sig \cup \battr{\adrg}$.
        \end{itemize}
            When the focus is on the intensional decision knowledge only, we omit the ABox, and call the tuple $\tup{\sig,\TBox,\adrg,\cbridge}$ intensional \dkb (\idkb).

%
%        In the following, we blur the distinction between attributes in a DMN
%        decision, and corresponding binary \fold predicates.

From the notational point of view, we can depict an \idkb $\aidkb = \tup{\sig,\TBox,\adrg,\cbridge}$ by slightly extending the \dmn notation for \drds as follows:
\begin{enumerate}
\item The knowledge base $\TBox$ is represented as a special business knowledge model.
\item Pictorially, this business knowledge model adopts the standard notation, using a small distinctive icon on the top right, and containing an indication about the bridge concept $\cbridge$.
\item There is an information requirement connecting the knowledge base to:
\begin{itemize}
  \item all input data of $\adrg$ that are also used by the knowledge base (thus highlighting the possible interaction between the input data and the background knowledge);
  \item all decisions of $\adrg$ that have at least one (free) input attribute different from all input data, and in common with the knowledge base (thus highlighting possible additional interactions with decision inputs that are not bound within the \drg).
\end{itemize}
\end{enumerate}
Notice that connecting a business knowledge model to input data of a \drg is forbidden by the standard. However, it is essential in our setting, so as to graphically highlight that the knowledge base may interact with certain input data.

\begin{example}
\label{ex:ship-dkb}
By considering our running example, the \dkb for the ship clearance
domain can be obtained by combining the knowledge base of Table~\ref{tab:shiptypes} with the \drg of Figure~\ref{fig:dmn-ex} (whose constitutive decisions are shown in Tables~\ref{tab:dmn-clearance} and \ref{tab:dmn-refuel}), using $\rel{Ship}$ as bridge concept. On the one hand, Table~\ref{tab:shiptypes} introduces different types
of ships, which can be modeled as subtype concepts of the generic concept of
ship, together with a set of axioms constraining the length, draft, and
capacity features depending on the specific subtype
(cf.~Example~\ref{ex:fol-axiom}). On the other hand,
Tables~\ref{tab:dmn-clearance} and \ref{tab:dmn-refuel} extend the signature of
Table~\ref{tab:shiptypes} with four additional features for ships, namely
certificate expiration and cargo, as well as the indication of whether a ship
can enter a port or not, and what its refuel area is. These two last features are produced as output of the
\drg in Figure~\ref{fig:dmn-ex}, and are in fact inferred by applying the \drg on a specific ship.

This \dkb is graphically shown in Figure~\ref{fig:bkm-ex} using the extended \drd notation. Notice how the diagram marks the possible interaction points of the knowledge base and the \drg.
\end{example}

\begin{figure}[t]
  \includegraphics[width=.5\textwidth]{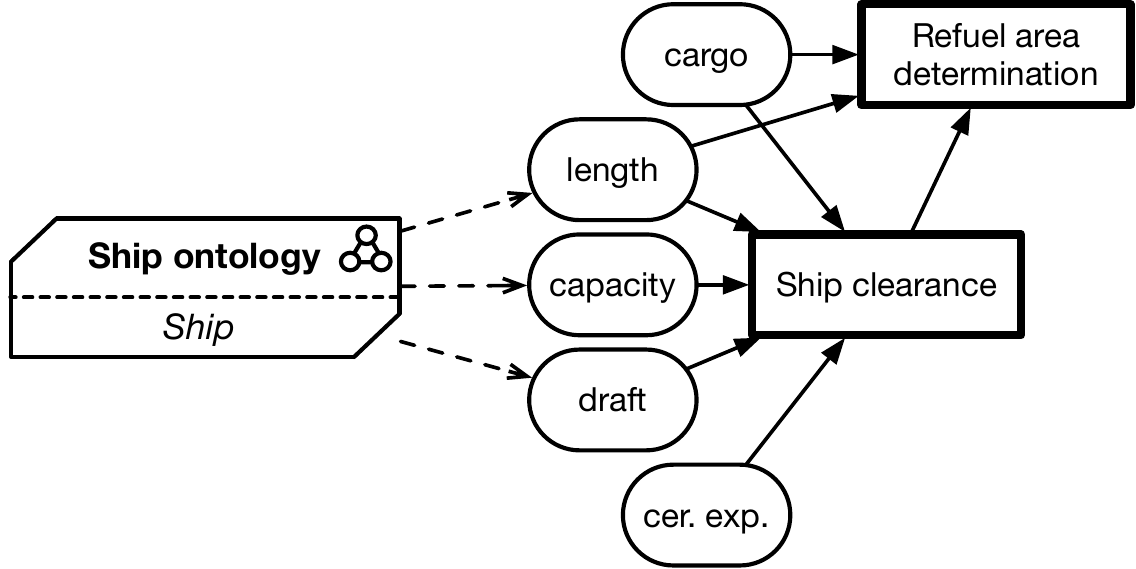}
  \caption{Graphical representation of an \idkb by extending the \dmn notation
   for \drgs}
  \label{fig:bkm-ex}
\end{figure}

\subsection{Formalizing \dkb{s}}
\label{sec:formalization}

From the formal point of view, the integration between a KB and a \drg is obtained by encoding the latter into \fold, consequently enriching the KB
with additional axioms that formally capture the decision logic.  We provide the encoding in this section. Obviously, to encode a \drg we first need to encode the decisions contained therein. To this purpose, we build on the logic-based formalization of DMN introduced in \cite{CDL16}. However, we cannot simply apply it as it is defined in \cite{CDL16}, since it does not follow the ``object-oriented'' approach required to interpret the application of a \drg in the presence of background knowledge. Specifically, that encoding formalizes decisions as formulae relating tuples of input values to tuples of output values, assuming no additional structure. In this work, we need to ``objectify" the approach in \cite{CDL16}, considering decisions as axioms that predicate on the features of a certain object, and that in particular postulate that whenever certain (input) features satisfy given conditions, then the object must be connected to certain other values through corresponding (output) features. This approach is useful to handle the integration with background knowledge, but also to simply interpret the interconnection of multiple decisions into a \drg, making our object-oriented formalization of \dmn decisions and \drgs is of independent interest.

Technically, we introduce an encoding $\tfol$ that translates an \idkb $\aidkb = \tup{\sig,\TBox,\adrg,\cbridge}$ into a corresponding \fold IKB $\tfol(\aidkb)$. The encoding can also be applied to a \dkb, translating its intensional part as before while leaving its extensional part unaltered, i.e., given a \dkb $\adkb = \tup{\sig,\TBox,\adrg,\cbridge,\ABox}$ such that $\tfol(\tup{\sig,\TBox,\adrg,\cbridge}) = \tup{\transl{\sig},\transl{\TBox}}$, we have $\tfol(\adkb) = \tup{\transl{\sig},\transl{\TBox},\ABox}$. We next describe how $\transl{\sig}$ and $\transl{\TBox}$ are actually constructed.

\subsubsection{Encoding of the Signature}
\label{sec:fold-signature}
The signature corresponds to the original signature of $\TBox$, augmented with a set of features that are obtained from the input data and the table attributes mentioned in $\adrg$.
  To avoid potential name clashes coming from repeated attribute names in different decisions, each attribute corresponds to a feature whose name is obtained by concatenating the name of such an attribute with the name of its decision. Given a decision $\dt$ and an attribute $\attr{a}$ of $\dt$, we use notation $\name{\dt}{a}$ to denote such a concatenated name.
Formally, we get:
\[
\transl{\sig} =
  \sig
  \cup \set{{P}/2 \mid {P} \in \get{\adrg}{\tin}}
  \cup \bigcup_{\dt \in \get{\adrg}{\dts}}
      \set{\name{\dt}{a}/2 \mid \attr{a} \in \get{\dt}{\tin} \cup \get{\dt}{\tin}}
\]
Each so-generated feature has its first component typed with $\univ$, and its second component typed with the datatype that is assigned by $\adrg$ to its corresponding input data/attribute.

\subsubsection{Encoding of the TBox}
\label{sec:fold-tbox}
The TBox extends the original axioms in $\TBox$ with additional axioms obtained by modularly encoding each decision and information flow of $\adrg$:
\[
\transl{\TBox} =
  \TBox
  \cup \bigcup_{\dt \in \get{\adrg}{\dts}} \big( \tfolb{\cbridge}(\dt) \big)
  \cup \bigcup_{f \in \get{\adrg}{\oimap}} \big( \tfol(f) \big)
\]
where the encoding $\tfolb{\cbridge}(\dt)$ of a decision table $\dt$ (parameterized by the bridge concept of $\adkb$), and the encoding $\tfol(f)$ of an information flow $f$, are detailed next.

\medskip
\noindent
\textbf{Encoding of decisions.} Let us consider the encoding $\tfolb{\cbridge}(\dt)$ of decision $\dt$, parameterized by bridge concept $C$. The encoding consists of the union
of axioms obtained by translating:
\begin{inparaenum}[\itshape (i)]
\item the input/output attributes of $\dt$;
\item the facet conditions or output ranges attached to such attributes;
\item the rules in $\dt$ (considering also priorities and default outputs).
\end{inparaenum}

%
%In the following, given a predicate $P \in \sigin{\dt} \cup \sigout{\dt}$, we denote by $\genattr{P}$ the attribute in
%$\get{\dt}{\tin} \cup \get{\dt}{\tout}$ from which $P$ has been obtained. In addition, given $m \in \set{1,\ldots,|\get{\dt}{\trules}|}$, we denote by $\sigoutr{\dt}{m} = \set{\predout{P}{k} \mid \predout{P}{k} \in \sigout{\dt},k=m}$ the subset of $\sigout{\dt}$ containing only the predicates associated to index $m$.

\medskip
\noindent
\textit{Encoding of attributes.~}
For each attribute $\attr{a} \in \get{\dt}{\tin} \cup \get{\dt}{\tout}$, the encoding $\tfolb{\cbridge}$ produces two axioms:
\begin{inparaenum}[\it (i)]
\item a typing formula $\forall x,y. \name{\dt}{a}(x,y) \limp \cbridge(x)$, binding the domain of the attribute to the bridge concept;
\item a functionality formula $\forall x,y,z. \name{\dt}{a}(x,y) \land \name{\dt}{a}(x,z) \limp y = z$, declaring that every object of the bridge concept cannot be connected to more than one value through $\name{\dt}{a}$.
\end{inparaenum}
If $\attr{a}$ is an input attribute, functionality  guarantees that the application of the decision table is unambiguous. If $\attr{a}$ is an output attribute, functionality simply captures that there is a single value present in an output cell of the decision. In general, multiple outputs for the same column may in fact be obtained when applying a decision but, if so, they would be still generated by different rules.

The exact same formalization does not only apply to the input attributes of a decision table, but also to the input data $\get{\adrg}{\tin}$ of the overall \drg $\adrg$.

\begin{example}
  Consider the \dkb in Example~\ref{ex:ship-dkb}. The typing and functionality for the \attr{Enter} input attribute for the refuel area determination decision (shown in Table~\ref{tab:dmn-refuel}, and for which we use compact name $\attr{Rad}$) are:
  \[
    \forall x,y. \name{\attr{Rad}}{Enter}(x,y) \limp \rel{Ship}(x)
    \qquad
    \forall x,y,z. \name{\attr{Rad}}{Enter}(x,y) \land \name{\attr{Rad}}{Enter}(x,z) \limp y = z
  \]
\end{example}

\smallskip
\noindent
\textit{Encoding of facet conditions and output ranges.~}
For each input attribute $\attr{a} \in \get{\dt}{\tin}$, function $\tfolb{\cbridge}$ produces a facet axiom imposing that the range of the feature must satisfy the restrictions imposed by the \sfeel condition $\get{\dt}{\infacet}(\attr{a})$. In formulae:
\[
\forall x,y. \name{\dt}{a}(x,y) \limp \tfol^y(\get{\dt}{\infacet}(\attr{a})),
\]
where, given an \sfeel condition $\cond$ and a variable $x$, function $\tfol^x(\cond)$ builds a unary \fold formula
that encodes the application of $\cond$ to $x$. This is defined as follows:
\[
  \tfol^x(\cond)
  \begin{cases}
    \mathit{true} & \text{if } \cond = ``\anycond"\\
    x \neq \cv{v} & \text{if } \cond = ``\mathtt{not(}\cv{v}\mathtt{)}"\\
    x = \cv{v}  & \text{if } \cond = ``\cv{v}"\\
    x \odot \cv{v} & \text{if } \cond = ``{\odot}\, \cv{v}"
      \text{ and } {\odot} \in \{{<},{>},{\leq},{\geq}\}\\
    x > \cv{v}_1 \land x < \cv{v}_2 & \text{if } \cond = ``(\cv{v}_1..\cv{v}_2)"
    \\
    \ldots & \text{(similarly for the other types of intervals)}\\
    % x > \cv{v}_1 \land x \leq \cv{v}_2
    % &
    % \text{if } \cond = ``(\cv{v}_1..\cv{v}_2]"
    % \\
    % x \geq \cv{v}_1 \land x < \cv{v}_2
    % &
    % \text{if } \cond = ``[\cv{v}_1..\cv{v}_2)"
    % \\
    % x \geq \cv{v}_1 \land x \leq \cv{v}_2
    % &
    % \text{if } \cond = ``[\cv{v}_1..\cv{v}_2]"
    % \\
    \tfol^x(\cond_1) \lor \tfol^x(\cond_2)
    & \text{if } \cond = ``\cond_1 \orcond \cond_2"\\
  \end{cases}
\]
The same mechanism is applied to the feature generated from each output attribute $\attr{b} \in \get{\dt}{\tout}$, reinterpreting its output range $\outrange(\attr{b}) = \tup{\cv{v}_1,\ldots,\cv{v}_n}$ as the \sfeel facet $``(\cv{v}_1\orcond\ldots\orcond\cv{v}_n)"$. Also in this case, the exact same formalization does not only apply to the attributes of a decision table, but also to the input data $\get{\adrg}{\tin}$ of the overall \drg $\adrg$.

\begin{example}
  Consider again the \dkb in Example~\ref{ex:ship-dkb}, and in particular the \attr{length} attribute in Table~\ref{tab:dmn-clearance} (for which we use the compact name $\attr{Sc}$). The facet \fold axioms for $\attr{length}$ is:
  \[
%      \forall x,y. \name{\attr{Sc}}{length}(x,y) \limp  \rel{Ship}(x).\qquad
      \forall x,y. \name{\attr{Sc}}{length}(x,y) \limp  y \geq 0.
  \]
\end{example}

\smallskip
\noindent
\textit{Encoding of rules.~}
Each rule is translated into an axiom expressing that, given an object:
\begin{description}
\item[\textit{\mdseries if}] the object has features for the input attributes of the rule whose values satisfy, attribute-wise, the \sfeel conditions associated by the rule to such input attributes,
\item[\textit{\mdseries then}] that object is also related, via output features, to the values associated by the rule to the output attributes.
\end{description}
Consider now a rule $r = \tup{\incond,\outval}$ in $\dt$. We first encode separately the input entry function $\incond$ and the output entry function $\outval$. similarly to the case of single \sfeel conditions, the encoding of $\incond$ and $\outval$ is parameterized by a variable $x$, representing an object to which the input/output entries are applied. Formally, we thus get:
  \[
    \tfol^x(\incond) = \bigwedge_{\attr{a} \in \get{\dt}{\tin}}
      \exists y. \Big(\name{\dt}{a}(x,y) \land \tfol^y(\incond(\attr{a}))\Big)
    \qquad
    \tfol^x(\outval) = \bigwedge_{\attr{b} \in \get{\dt}{\tout}}
      \exists y. \Big(\name{\dt}{b}(x,y) \land \tfol^y(\outval(\attr{b}))\Big)
  \]
  where $\tfol^y(\incond(\attr{a}))$ applies the encoding for \sfeel conditions defined before, on top of condition $\incond(\attr{a})$ and using variable $y$, obtained from $x$ by navigating the feature corresponding to $\attr{a}$. The selection of $y$ obtained via existential quantification is unambiguous, as features are functional.
  A similar observation holds for $\tfol^y(\outval(\attr{b}))$, noting that it simply produces a formula of the form $y = \cv{v}$, where $\cv{v}$ is the value assigned by rule $r$ to output attribute $\attr{b}$.

We now bind together the encoding of the rule premise and the rule conclusion into the overall encoding of rule $r$, which combines them into an implication formula. The body of this implication formula indicates when the rule trigger, which is partly based on the encoding of $\get{r}{\incond}$, and partly on the priority $\higherp$ (cf.~Section~\ref{sec:dmn}). %To do so, we define the following total order over $\get{\dt}{\trules}$ that is compatible with $\higherp$:
%\[\rindex: \get{\dt}{\trules} \rightarrow \set{1,\ldots,|\get{\dt}{\trules}|} \text{ s.t.~for every } r_1,r_2 \in \get{\dt}{\trules}, r_1 \higherp r_2 \text{ implies } \rindex(r_1) < \rindex(r_2).\]
Such a priority is in fact used to determine whether $r$ should really trigger on a given input object, or should instead stay quiescent because there is a higher-priority rule that triggers on the same object. With this notion at hand, we get:
\[
  \tfol(r) =
    \forall x.
      \tfol^x(\get{r}{\incond})
      \land \bigwedge_{r_2 \in \get{\dt}{\trules} \text{ and } r_2 \higherp r} \neg ( \tfol^x(\get{r_2}{\incond}) )
      \limp
      \tfol^x(\get{r}{\outval})
\]
%The last part of the implication body plays the role of a \emph{prioritization formula}, that is, it ensures that the body of $r$ evaluates to true only if there is no higher-priority rule $r_2 \higherp r$ whose body evaluates to true.
Due to the ``prioritization formula" used in the last part of the body, the overall encoding of all rules in the \drg is at most quadratic in the number of rules.
%
%key
%For every $m \in \set{1,\ldots,|\get{\dt}{\trules}|}$, given the $m$-rule $r_m = \tup{\incond_m,\outval_m} \in \trules$, function $\tfol_\cbridge$ produces:
%  \begin{align*}
%    &\forall x,\vec{y}.    \bigwedge_{\predin{P_j} \in \sigin{\dt}}
%    \hspace*{-.3cm}
%    \Big(P_j^i(x,y_j) \land \tfol^{y_j}\big(\genattr{\incond(P_j^i)}\big)\Big)
%    \limp
%    \hspace*{-.5cm}
%    \bigwedge_{\predout{P_k}{m} \in \sigoutr{\dt}{m}}
%    \hspace*{-.5cm}
%    \Big(\exists z_k. \predout{P_k}{m}  \land \tfol^{z_k}\big(\outval(\genattr{\predout{P_k}{m}})\big)\Big)
%  \end{align*}
This priority-preserving encoding correctly captures the semantics of rules irrespectively of which single hit indicator is used in $\dt$, possibly introducing some unnecessary conjuncts:
\begin{itemize}
  \item If $\dt$ semantically obeys to the unique hit strategy, then the input conditions of its rules are all mutually exclusive, and hence the prioritization formula is always trivially satisfied.
  \item If $\dt$ semantically obeys to the any hit strategy, then in case of multiple possible matches, all matching rules would actually return the same output values, and so the highest-priority matching rule can be safely selected.
  \item If $\dt$ adopts the priority hit policy, then the prioritization formula is actually needed to guarantee that the overall decision behaves according to what priority dictates.
  \end{itemize}

\begin{example}
Let us consider rule~2 in Table~\ref{tab:dmn-clearance}. Priority is, in this decision, irrelevant, as rules are indeed all non-overlapping. We can therefore ignore the prioritization formula, and simply get:
%, assuming that it has the highest priority.\footnote{This is a legal choice for $\rindex$ for rule~2, since all rules outputting $\cv{y}$ have equal priority, higher than the one of all rules outputting $\cv{n}$.}
%Using $\rel{Ship}$ as bridge concept, the encoding of such a rule is:
\[
\begin{array}{@{}r@{}l@{}}
\forall x.
%    \rel{Ship}(x)
    &\phantom{{}\land{}}
    (\exists e.\name{\attr{Sc}}{cerExp}(x,e) \land e > \cv{today})
    \land
    (\exists l.\name{\attr{Sc}}{length}(x,l) \land l < \cv{260})\\
    &{}\land
    (\exists d.\name{\attr{Sc}}{draft}(x,d) \land d < \cv{10})
    \land
    (\exists c. \name{\attr{Sc}}{capacity}(x,c) \land c < \cv{1000})
\limp
    \exists o. \name{\attr{Sc}}{enter}(x,o) \land o = \cv{y}.
\end{array}
\]
where $\rel{enter}_2$ is obtained from output attribute $\attr{enter}$ in the context of Rule~$2$.
\end{example}
%We close this section by arguing that our encoding can be seen as a sort of
%``objectification" of the encoding in \cite{CDL16}, where a tuple of values is
%now reified into an explicit object, together with corresponding predicates
%pointing to the different tuple components.

Since rules capture the intended input-output behavior of the decision, we have also to consider the case of default values for output attributes. Since default values are assigned when no rule triggers, we capture the ``default output behavior'' of decision $\dt$ as follows:
\[
    \forall x.
      \bigwedge_{r \in \get{\dt}{\trules}} \neg ( \tfol^x(\get{r}{\incond}) )
      \limp
      \bigwedge_{\attr{b} \in \get{\dt}{\tout} \text{ s.t.~} \get{\dt}{\outdef}(\attr{b}) \text{ is defined}}
        \Big(
          \exists y. \name{\dt}{\attr{b}}(x,y) \land y =  \get{\dt}{\outdef}(\attr{b})
        \Big)
\]
Note that it is not guaranteed that all attributes have a default value. If this is not the case, the formula above only binds those output facets whose corresponding attribute has a default value, leaving the other unspecified. This is perfectly compatible with the setting of \dkb{s}, which indeed work under incomplete information.

\medskip
\noindent
\textbf{Encoding of information flows.}
The encoding of information flows amounts to indicate that the source of an information flow feeds the target of the same information flow. This means that whenever a value is produced by the source, then this value is transferred into the target. Let $\tup{P,\attr{a}}$ be an information flow from input datum $P \in \get{\adrg}{\tin}$ to decision input attribute $\attr{a} \in \get{\dt}{\tin}$ for some decision table $\dt \in \get{\adrg}{\dts}$, and let $\tup{\attr{b},\attr{a}}$ be an information flow from decision output attribute $\attr{b} \in \get{\dt_1}{\tout}$ to decision input attribute $\attr{a} \in \get{\dt_2}{\tin}$ for some $\dt_1,\dt_2 \in \get{\adrg}{\dts}$. Then, we get:
\[
  \tfol(\tup{P,\attr{a}}) =
    \forall x,y.P(x,y) \limp \name{\dt}{\attr{a}}(x,y)
  \qquad
  \tfol(\tup{\attr{b},\attr{a}}) =
    \forall x,y.\name{\dt_1}{\attr{b}}(x,y) \limp \name{\dt_2}{\attr{a}}(x,y)
\]

\begin{example}
Consider the \drg of Figure~\ref{fig:dmn-ex}, observing that the information requirement connecting the input datum $\rel{length}$ and the clearance decision is due to the underlying information flow between such an input datum and the $\attr{length}$ attribute of the ship clearance decision in Table~\ref{tab:dmn-clearance}. Such an information flow is captured by the formula:
\[
\forall x,y. \rel{length(x,y)} \limp \name{\attr{Sc}}{\attr{length}}(x,y)
\]
\end{example}

\subsection{Reasoning Tasks}
\label{sec:tasks}

We now formally revisit and extend the main reasoning tasks introduced by \citeN{CDL16} DMN, considering here \dkbs equipped with complex decisions captured in a \drg. In the following, we generically refer to all such reasoning tasks as \emph{\dkb reasoning tasks}.

By considering a single decision table inside the \drg, we focus on the \emph{compatibility} of the decision with its \emph{policy hit}, considering the semantics of its rules in the context of the overall \dkb.
%
%\begin{inparaenum}[\it (i)]
%  \item capturing the input-output relationship induced by a decision;
%  \item checking whether the policy hit for a decision is indeed compatible with the rules therein;
%  \item verifying general properties, such as \emph{completeness} of a decision (the decision cover all possible input configurations) and \emph{relevance} of a rule (the rule is triggered with at least one input configuration).
%\end{inparaenum}
At the level of the whole \drg, we instead focus on the \emph{input-output relationship} induced by the \drg, arising from its internal decisions, information flows, and background knowledge. A related property is that of \emph{output coverage}, which checks whether all mentioned output values of the \drg can possibly be produced.
We also consider the two key properties of completeness and output determinability.
 \emph{Completeness} of a \dkb captures its ability of producing an overall output for the \drg for every configuration of values for its input data.
  Given a so-called \emph{template} describing a set of objects, \emph{output determinability} checks whether the template is informative enough to allow the \dkb determining an overall output for the \drg given an object that instantiates the template. Recall that a \drg has some decision tables marked as outputs of the \drg. In this light, an output of the \drg consists of the combination of an output for each one of its output tables.

%
%In the remainder of the section, we fix a \dkb
%$\adkb = \tup{\sig,\TBox,\adrg,\cbridge,\ABox}$, and an \idkb
%$\aidkb = \tup{\sig,\TBox,\adrg,\cbridge}$. We also consider a decision table $\dt \in \get{\adrg}{\dts}$.

\smallskip
\noindent
\textbf{Compatibility with ``unique hit".~}
Unique hit is declared in a decision table $\dt$ by setting $\get{\dt}{\thit} = \upol$, and dictates that for every input object, at most one rule of $\dt$ triggers. To check whether this is indeed the case, we introduce the problem of  \emph{compatibility with unique hit} as:
\begin{compactdesc}
\item[\textit{\mdseries Input:}] \idkb $\aidkb$, decision table $\dt \in \get{\get{\aidkb}{\adrg}}{\dts}$.
\item[\textit{\mdseries Question:}]	Is it the case that rules in $\get{\dt}{\trules}$ do not overlap, i.e., never trigger on the same input? Formally:
\[
\tfol(\aidkb) \askmodels
\bigwedge_{r_1,r_2 \in \get{\dt}{\trules} \text{ s.t. } r_1 \neq r_2}
\neg\exists x. \Big( \tfol^x(\get{r_1}{\incond}) \land \tfol^x(\get{r_2}{\incond}) \Big)
%\bigvee_{\attr{a} \in \get{\dt}{\tin}}
% \Big(
%    \tfol^x(\get{r_1}{\incond}) \limp \neg \tfol^x(\get{r_2}{\incond})
%\Big)
\]
\end{compactdesc}

\noindent
\textbf{Compatibility with ``any hit".~}
Any hit is declared in a decision table $\dt$ by setting $\get{\dt}{\thit} = \apol$,  and postulates that whenever multiple rules may simultaneously trigger, they need to agree on the
produced output. In this light, checking whether $\dt$ is compatible with this policy can be
directly reduced to the case of unique hit, but considering only those pairs of
rules in $\dt$ that \emph{differ} in at least one output value.

% let $\cv{c}$ be an object of type $\cbridge$ such that $\cbridge(\cv{c}) \in \ABox$, $P^o$ an output attribute in $\get{\dt}{\tout}$, and $\cv{v} \in \get{\dt}{\atype}(P^o)$ be a value for such an attribute. We say that \emph{$\instance$ assigns output $\cv{v}$ for $P^o$ to object $\cv{c}$} if $\tfol(\instance) \models P^o(\cv{c},\cv{v})$.

\noindent
\textbf{Compatibility with ``priority hit".~}
Priority hit is declared in a decision table $\dt$ by setting $\get{\dt}{\thit} = \ppol$,  and postulates that whenever multiple rules may simultaneously trigger, the one with the highest priority is selected. This is directly incorporated in the formalization of rules, so rules are by design compatible with priority hit. However, selecting this  policy may lead to the situation where a rule is \emph{masked by} an higher-priority rule, and hence would never trigger \cite{CDL16}. We thus consider that $\dt$ is compatible with the priority hit policy if none of its rules is masked. In this light, we introduce the problem of  \emph{compatibility with priority hit} as:
\begin{compactdesc}
\item[\textit{\mdseries Input:}] \idkb $\aidkb$, decision table $\dt \in \get{\get{\aidkb}{\adrg}}{\dts}$.
\item[\textit{\mdseries Question:}]	Is it the case that no rule in $\get{\dt}{\trules}$ is masked, i.e., there is at least one input object for which the rule triggers and no higher priority rule does? Formally:
\[
\tfol(\aidkb) \askmodels
\bigwedge_{r_1,r_2 \in \get{\dt}{\trules} \text{ s.t. } r_1 \higherp r_2}
\exists x. \Big( \tfol^x(\get{r_2}{\incond}) \land \neg \tfol^x(\get{r_1}{\incond}) \Big)
\]
\end{compactdesc}

\noindent
\textbf{I/O relationship.~}
A fundamental decision problem is to check whether the decision logic of a \dkb induces a certain input/output relationship for a given object, in the presence of an ABox that captures additional extensional data about the domain of interest (such as values assigned to that object for the input attributes of the \dkb). Specifically, the \emph{I/O relationship problem} for a decision is defined as:
\begin{compactdesc}
\item[\textit{\mdseries Input:}]
\begin{inparaenum}[\itshape (i)]
\item \dkb $\adkb$,
\item object $\cv{o} \in \univ$,
\item decision table $\dt \in \get{\get{\adkb}{\adrg}}{\odts}$,
\item output attribute $\attr{b} \in \get{\dt}{\tout}$,
\item value $\cv{v} \in \get{\dt}{\atype}(\attr{b})$.
\end{inparaenum}
\item[\textit{\mdseries Question:}] Is it the case that $\adkb$ relates object $\cv{c}$ to value $\cv{v}$ via feature $\name{\dt}{\attr{b}}$ ? Formally:
    \[
      \tfol(\adkb) \askmodels \name{\dt}{\attr{b}}(\cv{o},\cv{v})
    \]
\end{compactdesc}

\noindent
\textbf{Output coverage.~}
Output coverage refers to the I/O relationship induced by an \idkb, in this case focusing on the possibility of actually deriving a specific value for one of the output attributes of the \drg contained in the \idkb. If this is not possible, then it means that, due to the interplay between different decision tables and their information flows, as well as the contribution of the background knowledge, some output configurations are never obtained. Specifically, we defined the \emph{output coverage} problem as:
\begin{compactdesc}
\item[\textit{\mdseries Input:}]
\begin{inparaenum}[\itshape (i)]
\item \idkb $\aidkb$,
\item decision table $\dt \in \get{\get{\adkb}{\adrg}}{\odts}$,
\item output attribute $\attr{b} \in \get{\dt}{\tout}$,
\item value $\cv{v} \in \get{\dt}{\atype}(\attr{b})$.
\end{inparaenum}
\item[\textit{\mdseries Question:}] Does $\aidkb$ cover the possibility of outputting $\cv{v}$ for output attribute $\attr{b}$ of decision table $\dt$? Formally:
    \[
      \tfol(\aidkb) \askmodels \exists x,y. \name{\dt}{b}(x,y) \land y = \cv{v}
    \]
\end{compactdesc}

\begin{example}
Consider the \idkb $\aidkb_{ship}$ of our running example, in particular as defined in Example~\ref{ex:ship-dkb}. By focusing on the $\attr{RefuelArea}$ attribute of the output decision table \emph{refuel area determination} (cf.~\ref{tab:dmn-refuel}), we can see that value $\outdoor$ is not covered by $\aidkb_{ship}$. In fact, to produce such an output, rule $4$ should trigger, which in turn requires $\attr{length}$ and $\attr{cargo}$ to respectively be $>350$ and $>0.3$, and as well as $\attr{enter}$ to be $\cv{y}$. While the first two attributes are set by input data, the last is produced by the \emph{ship clearance} table, which is defined on the same input data, plus further ones (cf.~\ref{tab:dmn-clearance}). However, the only rule of \emph{ship clearance} that matches with the aforementioned conditions for $\attr{length}$ and $\attr{cargo}$, is in fact rule $9$, which however computes $\cv{n}$ for $\attr{enter}$, in turn falsifying the first condition of rule $4$ in \emph{refuel area determination}. This formally confirms the informal discussion of Section~\ref{sec:challenges}. Notice that this issue does not depend on the background knowledge, but on the (partial) incompatibility between the two decision tables.
\end{example}

\noindent
\textbf{Completeness.~}
Completeness asserts that the application of an \idkb to an arbitrary input object assigning values for the inputs of the \drg contained in the \idkb, is guaranteed to properly derive corresponding outputs. The \emph{\drg completeness problem} is then defined as follows:
\begin{compactdesc}
\item[\textit{\mdseries Input:}] \idkb $\aidkb$.
\item[\textit{\mdseries Question:}]	Is it the case that, for every object that assigns a value to each input of $\get{\get{\aidkb}{\adrg}}{\tin}$,  $\aidkb$ derives an output for each one of the output decision tables $\get{\get{\aidkb}{\adrg}}{\odts}$? Formally:
\[
\tfol(\aidkb) \askmodels
\forall x.
\Big(
\bigwedge_{P \in \get{\get{\aidkb}{\adrg}}{\tin}}
\exists y.P(x,y)
\Big)
\limp
\bigwedge_{\dt \in \get{\get{\aidkb}{\adrg}}{\odts}}
\bigwedge_{\attr{b} \in \get{\dt}{\tout}}
\exists y.\name{\dt}{b}(x,y)
\]
\end{compactdesc}

\noindent
\textbf{Output determinability.~}
Output determinability is a refinement of completeness. It amounts at checking whether, given a template describing a set of objects (encoded as a unary \fold formula), that template description is detailed enough to ensure that the \idkb
    properly derives the outputs of its \drg for every object that belongs to the template. This is in fact the only decision problem that only makes sense in the presence of background knowledge.
   Specifically, the \emph{output determinability problem} is defined as follows:
\begin{compactdesc}
\item[\textit{\mdseries Input:}] \idkb $\aidkb$, unary \fold formula $\varphi(x)$ over signature $\get{\aidkb}{\sig}$ (called \emph{template}).
\item[\textit{\mdseries Question:}]	Is it the case that, for every object that satisfies template $\varphi(x)$,  $\aidkb$ derives an output for each one of the output decision tables $\get{\get{\aidkb}{\adrg}}{\odts}$? Formally:
\[
\tfol(\aidkb) \askmodels
\forall x.
\varphi(x) \limp
\bigwedge_{\dt \in \get{\get{\aidkb}{\adrg}}{\odts}}
\bigwedge_{\attr{b} \in \get{\dt}{\tout}}
\exists y.\name{\dt}{b}(x,y)
\]
\end{compactdesc}
It is easy to see that completeness is a special case of output determinability, where the template simply describes objects that have all input data attached to them: $\varphi(x) = \bigwedge_{P \in \get{\get{\aidkb}{\adrg}}{\tin}}
\exists y.P(x,y)$.

\begin{example}
Consider again the \idkb $\aidkb_{ship}$ of Example~\ref{ex:ship-dkb}. We have already discussed in Section~\ref{sec:challenges} that to properly apply the decision logic formalized in $\aidkb_{ship}$, it is sufficient to know its type, cargo residuals, and certificate expiration date. This can be formalized as an output determinability problem, using as template the unary formula:
\[
  \varphi_{ship}(x) = \exists e,c,t. \rel{cerExp}(x,e) \land \rel{capacity}(x,c) \land \rel{stype}(x,t)
\]
\end{example}

\noindent
\textbf{DMN reasoning tasks.~}
We stress that, with the exception of output determinacy, all the decision problems identified here are relevant also when background knowledge is not present, and consequently a given \drg is interpreted under the assumption of complete information. In this case, compatibility with the different hit indicators, output coverage, and completeness, can all be captured as explained above, by simply setting $\TBox = \emptyset$. To account for I/O relationship, we have to put $\TBox = \emptyset$, and fix $\ABox$ to contain exactly the following facts:
\begin{inparaenum}[\itshape (i)]
\item a fact $\cbridge(\cv{o})$ for the selected object $\cv{o}$;
\item a set of facts of the form
  $\set{P(\cv{o},\cv{v}_j) \mid P \in \get{\get{\adkb}{\adrg}}{\tin}}$, denoting the
  assignment of input attributes for $\cv{o}$ to the corresponding values of interest, one per input data of the \drg.
\end{inparaenum}
In addition, all the identified decision problems can also be studied in the case of a single decision table $\dt$,  not immersed inside a \drg. This requires to construct the trivial \drg $\adrg_{\dt}$ that contains $\dt$ as the only decision table, marks it also as output table, and contains input data that exactly match (and feed via information flows) the input attributes of $\dt$. Properties of $\dt$ in the presence of background knowledge can then be assessed by studying a \dkb or \idkb that uses $\adrg_{\dt}$ as \drg.

%%% Local Variables:
%%% mode: latex
%%% TeX-master: "main"
%%% End:

\section{Reasoning on Decision Knowledge Bases}
\label{sec:translation}

While the translation from \dkb{s} to \fold presented in
Section~\ref{sec:formalization} provides a logic-based semantics for DKBs, it
does not give any insight on how to actually approach the different decision
problems of Section~\ref{sec:tasks}. In fact, none of such problems can be
solved in the general case of full \fold.  Specifically, decidability and
complexity of such reasoning tasks depend on the background knowledge and on
the decision component. Since the decision component comes with the fixed
\sfeel language and \drg structure, we approach this problem as follows.
First, we show that DMN decision tables written in \sfeel, and interconnected
in a \drg, can be encoded in \ALCHdd. Then, we show that all reasoning tasks
defined in Section~\ref{sec:tasks} can be reduced to (un)satisfiability of an
\ALCHdd concept w.r.t.~a KB consisting of the union of the background knowledge
with the \ALCHdd formalization of the \drg. This implies that all such reasoning
tasks can be carried out in \exptime, if the background knowledge is expressed
in \ALCHdd.

%We study the complexity of the different reasoning tasks introduced in
%Section~\ref{sec:tasks}, with the goal of isolating tractable cases, and to
%highlight those cases where the presence of background knowledge increases the
%complexity of such reasoning tasks w.r.t.~the pure DMN setting. As a side
%result, we also highlight which reasoning tasks are inherently intractable
%already in the DMN case, even when restricting the allowed datatypes.
%Obviously, all such reasoning tasks are undecidable in the general case where
%full \fold is employed. In our investigation, we hence consider DLs from the
%\dllite family, in particular \dllitebd[HN] and fragments thereof. Such logics
%capture widely established conceptual models like UML class diagrams, and are
%consequently well suited to represent background knowledge in our setting.

% . To this end, we first revisit the formalization introduced in
% Section~\ref{sec:formalization}, with the aim of understanding which
% expressiveness is needed in order to capture DMN decisions.
%
%
% . One the one hand, we are interested in understanding under which conditions
%
%
% This is done in two steps. First, we show that DMN decisions can be encoded
% into the Horn fragment of $\dllitebd[HN]$, i.e., into the $\dllitehd[HN]$
% logic of \cite{ACKRS-KRDB17-1}. Then, we leverage this encoding to identify
% tractable cases, or to show inherent intractability, of the considered
% reasoning tasks.

\subsection{Encoding \drgs in \ALCHdd}
\label{sec:translartion-alc}
%The first step towards complexity characterization is to revisit the approach proposed in Section~\ref{sec:formalization}, so as to understand which expressiveness is required to formalize DMN decisions. By inspecting the translation mechanism $\tfol$ from DMN to \fold, one would immediately obtain that DMN decisions can be encoded in $\dllitebd[HN]$.
%
% Here we provide a different encoding that, under some restrictions on the facet declarations, shows that $\dllitehd[HN]$ suffice for the translation. Technically,
We revisit the translation from \dkb{s} to \fold introduced in \ref{sec:formalization}, showing that the translation of \drgs can be reconstructed so as to obtain an \ALCHdd IKB.

%\subsubsection{Description of the Encoding}
 Given a bridge concept $\cbridge$ and a \drg $\adrg$, we introduce a translation function $\tdl_\cbridge$ that encodes $\dt$ into the corresponding \ALCHdd IKB $\tdl_\cbridge(\adrg) = \tup{\sig_\adrg,\TBox_\adrg}$, using $\cbridge$ to provide a context for the encoding. The signature is obtained as in Section~\ref{sec:fold-signature}. The encoding of $\TBox_\adrg$ reconstructs that of Section~\ref{sec:formalization}, and in fact deals with input data and information flows of the $\adrg$, as well as input/output attributes, facets, and rules of decision tables $\get{\adrg}{\dts}$.

\smallskip
\noindent
\textbf{Encoding of attributes and input data.~}
For each decision table $\dt \in \get{\adrg}{\dts}$, and each attribute $\attr{a} \in \get{\dt}{\tin} \cup \get{\dt}{\tout}$,  encoding $\tdl_\cbridge$ produces the typing axiom $ \SOMET{\name{\dt}{a}} \ISA \cbridge$. The same holds for all input data $\get{\adrg}{\tin}$.  Functionality is not explicitly asserted, since \ALCHdd features are functional by default.

\smallskip
\noindent
\textbf{Encoding of facet conditions.~}
For each decision table $\dt \in \get{\adrg}{\dts}$, and each input attribute $\attr{a} \in \get{\dt}{\tin}$,  encoding $\tdlb{\cbridge}$ produces a derived datatype declaration of the form
\[
  \SOMET{\name{\dt}{a}} \ISA \tdl^{\name{\dt}{a},\atype(\attr{a})}(\get{\dt}{\infacet}(\attr{a}))
\]
where, given an \sfeel condition $\cond$, a facet $P$, and a datatype $type$, function $\tdl^{P,type}$ produces an \ALCHdd concept capturing objects that have an outgoing facet of type $P$, whose range satisfies $\cond$. This is defined as follows:
\[
\tdl^{P,type}(\cond) =
\begin{cases}
  \top
    & \text{if } \cond = ``\anycond"\\
  \neg \SOMET{P.type}[=_\cv{v}]
    & \text{if } \cond = ``\mathtt{not(}\cv{v}\mathtt{)}"\\
  \SOMET{P.type}[=_\cv{v}]
    & \text{if } \cond = ``\cv{v}"\\
  \SOMET{P.type}[\odot_\cv{v}]
    & \text{if } \cond = ``\odot\cv{v}"  \text{ and } \odot \in \set{<,>,\leq,\geq}\\
  \SOMET{P.type}[>_{\cv{v}_1} \land <_{\cv{v}_2}]
    & \text{if } \cond = ``(\cv{v}_1..\cv{v}_2)"\\
\ldots &\text{(similarly for the other types of intervals)}\\
%x > \cv{v}_1 \land x \leq \cv{v}_2
%&
%\text{if } \cond = ``(\cv{v}_1..\cv{v}_2]"
%\\
%x \geq \cv{v}_1 \land x < \cv{v}_2
%&
%\text{if } \cond = ``[\cv{v}_1..\cv{v}_2)"
%\\
%x \geq \cv{v}_1 \land x \leq \cv{v}_2
%&
%\text{if } \cond = ``[\cv{v}_1..\cv{v}_2]"
%\\
  \tdl^{P,type}(\cond_1) \OR \tdl^{P,type}(\cond_2)
    & \text{if } \cond = ``\cond_1 \orcond \cond_2"\\
\end{cases}
\]
The same encoding is applied by $\tdlb{\cbridge}$ to each input data $P \in \get{\adrg}{\tin}$ with its facet $\get{\adrg}{\infacet}(P)$, and, for every decision table $\dt \in \get{\adrg}{\dts}$, to each output attribute $\attr{b} \in \get{\dt}{\tout}$ with its range $\get{\dt}{\outrange}(\attr{b})$.

\begin{example}
  Consider the $\attr{length}$ attribute of the \emph{ship clearance} decision
  table (cf.~\tablename~\ref{tab:dmn-clearance}). With $\rel{Ship}$ as bridge
  concept, the typing and facet $\ALCHdd$ formulae for $\attr{length}$ are:
  \[
    \SOMET{\rel{length}}  \ISA \rel{Ship}
    \qquad
    \SOMET{\rel{length}}  \ISA  \SOMET{\rel{length.real}}[>_0]
  \]
\end{example}

\noindent
\textbf{Encoding of rules.} Consider a decision table $\dt$, and one of its rules $r = \tup{\incond,\outval}$. The encoding of $\incond$ (resp., $\outval$) constructs an \ALCHdd concept that has all features mentioned by the input (resp., output) attributes, restricted so as to satisfy the corresponding input condition (resp., output value) imposed by $\incond$ (resp., $\outval$):
\[
  \tdlb{\cbridge}(\incond) =
    \bigsqcap_{\attr{a} \in \get{\dt}{\tin}}\tdl^{\name{\dt}{a},\get{\dt}{\atype}(\attr{a})}
    (\incond(\attr{a}))
  \qquad
   \tdlb{\cbridge}(\outval) =
    \bigsqcap_{\attr{b} \in \get{\dt}{\tout}}\tdl^{\name{\dt}{b},\get{\dt}{\atype}(\attr{b})}
    (\incond(\attr{b}))
\]
We combine these two encodings into a global encoding of rule $r$, imposing that the rule indeed triggers only if no higher-priority rule triggers:
\[
\tdlb{\cbridge}(r) =
  \tdlb{\cbridge}(\get{r}{\incond})
  \sqcap
  \bigsqcap_{r_2 \in \get{\dt}{\trules} \text{ and } r_2 \higherp r}
    \neg \tdlb{\cbridge}(\get{r_2}{\incond})
  \ISA
   \tdlb{\cbridge}(\get{r}{\outval})
\]

\begin{example}
  Consider the \emph{ship clearance} decision table, referred by name
  $\attr{Sc}$. In particular, consider rule~2 of this decision table, as shown
  \tablename~\ref{tab:dmn-clearance}. By assuming that this is the top-priority
  rule, it is encoded in $\ALCHdd$ as:
\[
  \begin{array}{@{}l}
    \SOMET{\name{\attr{Sc}}{cerExp}.\rel{real}}[>_{\cv{today}}] \AND
    \SOMET{\name{\attr{Sc}}{length}.\rel{real}}[<_{\cv{260}}] \\
    {}\AND \SOMET{\name{\attr{Sc}}{draft}.\rel{real}}[<_{\cv{10}}] \AND
    \SOMET{\name{\attr{Sc}}{cap}.\rel{real}}[<_{\cv{1000}}]
    ~\ISA~
    \SOMET{\name{\attr{Sc}}{enter}.\rel{string}}[=_{\cv{Y}}]
  \end{array}
\]
\end{example}
We also have to handle the generation of default values, when no rule in $\dt$ triggers. This is captured by the following, additional axiom:
\[
    \bigsqcap_{r \in \get{\dt}{\trules}} \neg \tdlb{\cbridge}(\get{r}{\incond})
      ~\ISA~
      \bigsqcap_{\attr{b} \in \get{\dt}{\tout} \text{ s.t.~} \get{\dt}{\outdef}(\attr{b}) \text{ is defined}}
        \Big(
          \SOMET{\name{\dt}{\attr{b}}}[=_{\get{\dt}{\outdef}(\attr{b})}]
        \Big)
\]

\smallskip
\noindent
\textbf{Encoding of information flows.}
Let $\tup{P,\attr{a}}$ be an information flow from input datum $P \in \get{\adrg}{\tin}$ to decision input attribute $\attr{a} \in \get{\dt}{\tin}$ for some decision table $\dt \in \get{\adrg}{\dts}$, and let $\tup{\attr{b},\attr{a}}$ be an information flow from decision output attribute $\attr{b} \in \get{\dt_1}{\tout}$ to decision input attribute $\attr{a} \in \get{\dt_2}{\tin}$ for some $\dt_1,\dt_2 \in \get{\adrg}{\dts}$.
 Their \ALCHdd encoding consists of the following facet inclusion assertions:
\[
   \tdlb{\cbridge}(\tup{P,\attr{a}})  =
    P ~\ISA \name{\dt}{\attr{a}}
  \qquad
  \tdlb{\cbridge}(\tup{\attr{b},\attr{a}})  =
    \name{\dt_1}{\attr{b}} ~\ISA \name{\dt_2}{\attr{a}}
\]

\smallskip
\noindent
\textbf{Correctness of the encoding.}
Thanks to the fact that \ALCHdd can be seen as a well-behaved fragment of \fold, we can
directly establish that the \ALCHdd encoding of \drgs properly reconstructs the original \fold encoding.
% in the following precise sense.

\begin{theorem}
\label{thm:equivalence}
For every \drg $\adrg$, and every (bridge) concept $\cbridge$, we have that the \fold IKB $\tfolb{\cbridge}(\adrg)$ is \emph{logically equivalent} to the \ALCHdd IKB $\tdlb{\cbridge}(\adrg)$.
\end{theorem}
\begin{proof}Direct by definition of the encodings $\tfolb{\cbridge}$ and $\tdlb{\cbridge}$, noting that, once \ALCHdd IKB $\tdlb{\cbridge}(\dt)$ is
represented in \fold using the standard \fold encoding of \ALCHdd, it becomes identical to the \fold IKB $\tfolb{\cbridge}(\dt)$.
\end{proof}

\subsection{Reasoning over \ALCHdd Decision Knowledge Bases}

By exploiting the \ALCHdd encoding of a \drg, we have now the possibility of
studying if and how the different reasoning tasks introduced in
Section~\ref{sec:tasks} can be effectively carried out in the case where the
background knowledge is also represented as an \ALCHdd (I)KB.

We say that an \idkb $\aidkb$ is an \emph{\ALCHdd \idkb} if its TBox
$\get{\adkb}{\TBox}$ is an \ALCHdd TBox. As for a $\dkb$ $\adkb$, we require the
same, and also that and its ABox $\get{\adkb}{\ABox}$ is an \ALCHdd ABox.  Given
an \ALCHdd \dkb $\adkb$, we extend the encoding function
$\tdlb{\get{\adkb}{\cbridge}}$ introduced in Section~\ref{sec:translartion-alc}
so as to make it applicable over the entire \dkb, as follows:
$\tdlb{C}(\adkb) = \tup{\get{\adkb}{\sig} \cup \sig', \get{\adkb}{\TBox} \cup
 \TBox',\get{\adkb}{\ABox}}$, where
$\tup{\sig',\TBox'} = \tdlb{\cbridge}(\get{\adkb}{\adrg})$ (similarly for an
\ALCHdd IKB). With these notions at hand, we show the following.

\begin{theorem}
  \label{thm:reasoning}
  In the case of \ALCHdd \dkb{s}, all \dkb reasoning tasks can be reduced to
  standard \ALCHdd reasoning tasks.
\end{theorem}
\begin{proof}
We show, for each \dkb reasoning task, how it can be reduced to a polynomial
number of \ALCHdd concept (un)satisfiability or instance checking tests
w.r.t.\ an \ALCHdd KB.

\smallskip
\noindent\textit{Compatibility with ``unique hit".~} We use the following
algorithm, relying on \ALCHdd satisfiability checking. In the following
algorithm, the usage of $\higherp$ is not needed for correctness, but actually
matters to reduce the number of checks (being the notion of overlap symmetric).

\begin{lstlisting}[frame=none,mathescape=true]
boolean compatibleWithU(IDKB $\aidkb$, Table $\dt \in \get{\aidkb}{\dts}$) {
  for each $r_1,r_2 \in \get{\dt}{\trules}$ such that $r_1 \higherp r_2$ {
    if $\tdlb{\get{\aidkb}{\cbridge}}(\get{r_1}{\incond}) \sqcap \tdlb{\get{\aidkb}{\cbridge}}(\get{r_2}{\incond})$ is satisfiable w.r.t. $\tdlb{\get{\aidkb}{\cbridge}}(\aidkb)$
      return false;
  }
  return true;
}
\end{lstlisting}

\noindent\textit{Compatibility with ``any hit".~} We use exactly the same
algorithm used for compatibility with ``unique hit", with the only difference
that in line 2 we add $\get{r_1}{\outval} = \get{r_2}{\outval}$ as a further
condition, to ensure that the output values produced by $r_1$ and $r_2$
coincide.

\smallskip
\noindent\textit{Compatibility with ``priority hit".~} We use the following
algorithm, relying on \ALCHdd unsatisfiability checking.

\begin{lstlisting}[frame=none,mathescape=true]
boolean compatibleWithP(IDKB $\aidkb$, Table $\dt \in \get{\aidkb}{\dts}$) {
  for each $r_1,r_2 \in \get{\dt}{\trules}$ such that $r_1 \higherp r_2$ {
    if $\neg \tdlb{\get{\aidkb}{\cbridge}}(\get{r_1}{\incond}) \sqcap \tdlb{\get{\aidkb}{\cbridge}}(\get{r_2}{\incond})$ is unsatisfiable w.r.t. $\tdlb{\get{\aidkb}{\cbridge}}(\aidkb)$
      return false;
  }
  return true;
}
\end{lstlisting}

\noindent
\textit{I/O relationship.~} We use the following algorithm, relying on \ALCHdd
instance checking.
\begin{lstlisting}[frame=none,mathescape=true]
boolean IORelationship(DKB $\adkb$, Table $\dt \in \get{\aidkb}{\odts}$, Object $\cv{o} \in \univ$,
                       Attribute $\attr{b} \in \get{\dt}{\tout}$, Value  $\cv{v} \in \get{\dt}{\atype}(\attr{b})$) {
    return $\tup{\cv{o},\cv{v}}$ instance of $\name{\dt}{\attr{b}}$ w.r.t. $\tdlb{\get{\adkb}{\cbridge}}(\adkb)$;
}
\end{lstlisting}

\noindent
\textit{Output coverage.~}  We use the following algorithm, relying on \ALCHdd
satisfiability checking.
\begin{lstlisting}[frame=none,mathescape=true]
boolean CoversOutput(IDKB $\aidkb$, Table $\dt \in \get{\aidkb}{\odts}$,
                     Attribute $\attr{b} \in \get{\dt}{\tout}$, Value  $\cv{v} \in \get{\dt}{\atype}(\attr{b})$) {
    return $\SOMET{\name{\dt}{\attr{b}}[=_\cv{v}]}$ is satisfiable w.r.t. $\tdlb{\get{\aidkb}{\cbridge}}(\aidkb)$;
}
\end{lstlisting}

\smallskip
\noindent\textit{Completeness.} We use the following algorithm, relying on \ALCHdd satisfiability checking.

\begin{lstlisting}[frame=none,mathescape=true]
boolean Complete(IDKB $\aidkb$) {
  for each  $\dt \in \get{\get{\aidkb}{\adrg}}{\odts}$ {
    for each $\attr{b} \in \get{\dt}{\tout}$ {
      if $\bigsqcap_{P \in \get{\get{\aidkb}{\adrg}}{\tin}} \SOMET{P} \sqcap \neg \SOMET{\name{\dt}{b}}$ is satisfiable w.r.t. $\tdlb{\get{\aidkb}{\cbridge}}(\aidkb)$
        return false;
    }
  }
  return true;
}
\end{lstlisting}

\smallskip
\noindent\textit{Output determinability.} We use the following algorithm, relying on \ALCHdd satisfiability checking. Obviously, we consider templates  described by \ALCHdd concepts.

\begin{lstlisting}[frame=none,mathescape=true]
boolean DeterminesOutput(IDKB $\aidkb$, $\ALCHdd$ Concept $\Phi$) {
  for each  $\dt \in \get{\get{\aidkb}{\adrg}}{\odts}$ {
    for each $\attr{b} \in \get{\dt}{\tout}$ {
      if $\Phi \sqcap \neg \SOMET{\name{\dt}{b}}$ is satisfiable w.r.t. $\tdlb{\get{\aidkb}{\cbridge}}(\aidkb)$
        return false;
    }
  }
  return true;
}
\end{lstlisting}
It is straightforward to check that all the presented algorithms correctly reconstruct the corresponding \fold decision problems.
\end{proof}

\begin{example}
  As discussed in Example~\ref{ex:ship-ontology}, the ship ontology in
  \tablename~\ref{tab:shiptypes} can be formalized in \ALCHdd. Hence, the
  maritime security \dkb of Example~\ref{ex:ship-dkb} is actually an \ALCHdd
  \dkb. Thanks to Theorem~\ref{ex:dmn}, standard \ALCHdd reasoning tasks can
  then be used to carry out all the introduced reasoning tasks over such a
  \dkb.
\end{example}

Thanks to Theorems~\ref{thm:alcdd} and~\ref{thm:reasoning}, we obtain two
additional key results. The first result characterizes the complexity of $\dkb$
reasoning in the case of \ALCHdd \dkb{s}.

\begin{corollary}
  \label{cor:complexity}
  All \dkb reasoning tasks over \ALCHdd \dkb{s} can be decided in \exptime.
\end{corollary}

Notice that the complexity of reasoning in \dkb{s} that employ different
ontology languages to capture the background knowledge, depends on the actual
ontology language of choice, considering that the encoding of \drgs brings an
\ALCHdd component. In general, since \ALCHdd datatypes come with unary
predicates, our approach naturally lends itself to be combined with OWL\,2
ontologies, and the \ALCHdd encoding of \drgs can in fact be directly
represented in OWL\,2.

\begin{corollary}
  \label{cor:concrete}
  All \dkb reasoning tasks over \ALCHdd \dkb{s} can be tackled by standard
  OWL\,2 reasoners.
\end{corollary}
This is an important observation, since one can then resort to state-of-the-art
reasoners for OWL\,2 that have been developed and optimized over the years
\cite{TsHo06,SiPa06,ShMH08}.
%
% there are no well-established reasoners
% for DLs extended with datatypes, while the literature flourishes of very
% effective reasoners for standard \ALC ontologies.
%
When considering reasoning tasks that only focus on intensional knowledge, that
is, all reasoning tasks introduced in Section~\ref{sec:tasks} with the
exception of I/O relationship, it is also possible to rely on reasoners for
OWL\,2 TBoxes that \emph{do not support datatypes}. In fact, we can reconstruct
the  technique introduced by \citeN[Theorem~2.14]{Lutz02c} to encode away unary
concrete domains, so as to compile away datatypes from IDKBs, finally obtaining
a pure \ALCH TBox. However, this requires to exhaustively apply datatype
reasoning during the compilation process, Hence, it remains open whether this
introduces an effective improvement over full OWL\,2 reasoners, which typically
apply datatype reasoning lazily, only when needed.

A second open problem is to show whether \drgs can be
encoded in weaker ontology languages, so as to obtain more refined complexity
bounds for the different \dkb reasoning tasks. The main difficulty here stems
from the fact that extensions of lightweight DLs with datatypes have been so
far much less investigated than their corresponding expressive counterparts. In
particular, currently known lightweight DLs with datatypes are
% either
equipped with an ontology language that is too weak to encode \drgs
\cite{ArKR12,SaCa12}.
% , or are actually rich enough to encode \drgs, but come with additional
% restrictions on the datatypes that are not compatible with \sfeel
% \cite{ACKRS-KRDB17-1}.

%%% Local Variables:
%%% mode: latex
%%% TeX-master: "main"
%%% End:

\section{Related Work}
\label{sec:related}

To the best of our knowledge, this work is the first approach that combines
\dmn \drgs with background knowledge, building on the preliminary results
obtained in \cite{CDMM17} for the case of single decision tables. In addition,
it is also the first approach that considers reasoning tasks over \dmn \drgs,
even without considering the contribution of background knowledge. In this
light, it can be considered as a natural extension of the formalization effort
carried out in \cite{CDL16}.

Reasoning on single decision tables has instead attracted a lot of interest in
the literature, an interest recently revived by the introduction of the \dmn
standard. In particular, reasoning tasks that aim at assessing completeness,
consistency and redundancy of decision tables, are widely
recognized~\cite{CodasylReport}. The literature flourishes of ad-hoc,
algorithmic techniques to account for such decision problems, considering
specific datatypes. In particular, one long-standing line of research
comprises techniques to reason about decision tables whose attributes are
either boolean or enumerative (i.e., categorical)
\cite{Pawlak1987,HoCh95,ZaLe97}. Some of these techniques have been
actually implemented inside well-known tools like
Prologa~\cite{Vanthienen1994,Vanthienen1998}.

The main drawback of these approaches is that they do not directly account for
numerical datatypes: in the presence of conditions expressing numerical
intervals, they require to restructure their corresponding rules so as to
ensure that all intervals are disjoint. \cite{CDL16} introduces ad-hoc
algorithmic techniques based on a geometric interpretation of rules, and shows
that such techniques outperform previous approaches, while being able to
naturally handle numerical domains. The DMN component of Signavio\footnote{\url{https://www.signavio.com/}} detects
overlapping and missing rules by natively dealing with numerical data
types. However, the actual algorithms used to conduct such checks have not been
disclosed. OpenRules\footnote{\url{http://openrules.com/}} builds instead on constraint satisfaction techniques to
analyze rules containing numerical attributes.

Differently from all these approaches, we consider here full \dmn \drgs in the
presence of background knowledge. This richer setting also demands a wider and
more sophisticated set of reasoning tasks, going beyond completeness and
consistency of single decision tables. For such advanced reasoning tasks, we do
not develop ad-hoc algorithmic techniques, but instead rely on a fully
automated encoding of the input specification, and of the tasks of interest,
into standard reasoning tasks for the DL \ALCHdd.  Efficient, state-of-the-art
reasoners have been devised for expressive DLs, such as OWL~2
\cite{HoKS06,W3Crec-OWL2-Overview}, that fully capture \ALCHdd
\cite{TsHo06,SiPa06,ShMH08}, setting the baseline for a future experimental
evaluation of the techniques presented in this paper, considering real and
synthetic data. In addition, by inspecting the proof of
Theorem~\ref{thm:reasoning}, it is easy to see that all the presented
algorithms can be easily modified so as to return the actual, involved rules
whenever a property is not satisfied.

% Beside correctness-related checks, many algorithms have been developed
% towards the simplification of decision tables via rule merging
% ~\cite{Pollack1965,Shwayder1975,Maes1980,Vanthienen1996}.  The Pollack's
% algorithm~\cite{Pollack1965} selects two rules that have the same output and
% coincide on all inputs but one. Every time that such a pair of rules is
% identified, they are merged into a single rule by doing the union of the sets
% of values in the two cells where the difference occurs (all other cells
% remain the same in the merged rule).  Shwayder~\cite{Shwayder1975} proposes
% an optimization of the Pollack's algorithm applicable in the case where
% certain rules do not depend on all attributes, but only on a subset of
% them. Maes~\cite{Maes1980} proposes further optimizations in cases where
% there are logical relations between the attributes (e.g., if two attributes
% are true, then a third attribute is also true).  The latter approach
% specifically optimizes the order in which the attributes are scanned (to
% identify possible rules to be merged) so as to obtain a minimum set of rules
% after simplification. In this approach, only one pair of rules is merged at a
% time.  Vanthienen et al.~\cite{Vanthienen1996} extend the latter approach by
% considering situations where groups of more than two rules can be merged
% together into a smaller set of rules.

From the knowledge representation point of view, this work touches the
widely studied, and still debated, problem of integrating rules and
ontologies. This problem has been approached in different ways, depending on
the expressiveness of the rule and of the ontology languages
\cite{DEIK09,KrMH11}. On the one hand, several proposals have been devised to integrate rules and ontologies by defining suitable ``hybrid'' semantics
\cite{MoRo10}, or by considering rules accessing ontologies as an external
knowledge component \cite{EKRS17}. On the other hand, ``controlled'' forms of
rules have been integrated with ontologies by reformulating them as additional
ontological axioms \cite{KrMH11}. Our contribution belongs to the latter family, thanks to the interesting trade-off between
expressiveness and simplicity offered by the \dmn \sfeel language.

%%% Local Variables:
%%% mode: latex
%%% TeX-master: "main"
%%% save-place: t
%%% End:

\section{Conclusions}
\label{sec:conclusions}

In this work, we have provided a threefold contribution to the area of decision
management, recently revived by the introduction of the DMN OMG standard.
First, we have introduced \emph{decision knowledge bases} (\dkb{s}) as a
conceptual framework to integrate DMN complex decisions with background
knowledge, expressed as a description logic (DL) knowledge base.  On top of this
conceptual framework, we have then introduced key reasoning tasks to ascertain
the correctness of a \dkb.  Second, we have provided a logic-based formalization
of \dkbs and their corresponding reasoning tasks, using multi-sorted FOL
equipped with datatypes.  Third, we have focused our attention on the
interesting case where the background knowledge is expressed using \ALCHdd, an
extension of the well-known description logic \ALC with multiple datatypes, and
a sublanguage of the standard ontology language OWL\,2.  In this
setting, we have shown that all the aforementioned reasoning tasks are
decidable in \exptime, and lend themselves to be carried out using standard
reasoners for expressive description logics. On the way of proving this result,
we have shown that TBox and ABox reasoning for \ALCH extended with multiple
datatypes stays within \exptime, which is of independent interest.
% for \ALC.

These three contributions pave the ways towards a concrete implementation of
the presented framework and techniques. We plan to realize this implementation
and to consequently carry out an experimental evaluation by considering not
only full \dkbs, but also \dkbs consisting of a single decision table, as well
as complex decisions without background knowledge, so as to better identify the
sources of complexity, and to see how well a general approach of this form
compares with the ad-hoc algorithms developed in the literature. In spite of
the \exptime upper bound for reasoning on \dkbs, we believe that an effective,
scalable implementation is actually at reach, thanks to the availability of
solid, optimized reasoners for OWL\,2.

In addition to the implementation effort, we are interested in refining our
complexity analysis, in particular aiming at tighter bounds on the complexity
caused by the decision component.  Specifically, we plan to systematically
study how lightweight DLs equipped with datatypes \cite{SaCa12,ArKR12}, for
which currently the ontology language is too weak to capture complex DMN
decision tables, can be extended, focusing on their ability of dealing with
datatypes and features.  We would like to single out more precisely the
complexity brought in by a DMN complex decision table, with the aim of
capturing more complex forms of tables, without compromising the low
computational complexity of reasoning in lightweight DLs (\ACz in the size of
the data).  As a consequence, this would pave the way towards lightweight
\dkb{s}.

%%% Local Variables:
%%% mode: latex
%%% TeX-master: "main"
%%% End:

\paragraph{\textbf{Acknowledgements.}~}
This research is partly supported by the Estonian Research Council Grant
IUT20-55,
by the project ``Reasoning and Enactment for Knowledge-Aware Processes''
(REKAP), which is funded through the 2017 call issued by the Research Committee
of the Free University of Bozen-Bolzano, and
by the Euregio Interregional Project Network IPN12 ``Knowledge-Aware
Operational Support'' (KAOS), which is funded by the ``European Region
Tyrol-South Tyrol-Trentino'' (EGTC) under the first call for basic research
projects and by the Free University of Bozen-Bolzano.

\bibliographystyle{acmtrans}
%\bibliography{string-medium,local}
\bibliography{main-bib}

\begin{thebibliography}{}

\bibitem[\protect\citeauthoryear{Artale, Kontchakov, and Ryzhikov}{Artale
  et~al\mbox{.}}{2012}]{ArKR12}
{\sc Artale, A.}, {\sc Kontchakov, R.}, {\sc and} {\sc Ryzhikov, V.} 2012.
\newblock \textit{DL-Lite} with attributes and datatypes.
\newblock In {\em Proc.\ of the 20th Eur.\ Conf.\ on Artificial Intelligence
  (ECAI)}. Frontiers in Artificial Intelligence and Applications, vol. 242. IOS
  Press, 61--66.

\bibitem[\protect\citeauthoryear{Baader, Calvanese, McGuinness, Nardi, and
  Patel-Schneider}{Baader et~al\mbox{.}}{2007}]{BCMNP07}
{\sc Baader, F.}, {\sc Calvanese, D.}, {\sc McGuinness, D.}, {\sc Nardi, D.},
  {\sc and} {\sc Patel-Schneider, P.~F.}, Eds. 2007.
\newblock {\em The Description Logic Handbook: {T}heory, Implementation and
  Applications\/}, 2nd ed.
\newblock Cambridge University Press.

\bibitem[\protect\citeauthoryear{Baader and Sattler}{Baader and
  Sattler}{2000}]{BaSa00}
{\sc Baader, F.} {\sc and} {\sc Sattler, U.} 2000.
\newblock Tableau algorithms for description logics.
\newblock In {\em Proc.\ of the 9th Int.\ Conf.\ on Automated Reasoning with
  Analytic Tableaux and Related Methods (TABLEAUX)}. Lecture Notes in
  Artificial Intelligence, vol. 1847. Springer, 1--18.

\bibitem[\protect\citeauthoryear{Bao et~al\mbox{.}}{Bao
  et~al\mbox{.}}{2012}]{W3Crec-OWL2-Overview}
{\sc Bao, J.} {\sc et~al\mbox{.}} 2012.
\newblock {OWL~2} {W}eb {O}ntology {L}anguage document overview (second
  edition).
\newblock {W3C} {R}ecommendation, World Wide Web Consortium. Dec.
\newblock Available at \protect\url{http://www.w3.org/TR/owl2-overview/}.

\bibitem[\protect\citeauthoryear{Batoulis, Meyer, Bazhenova, Decker, and
  Weske}{Batoulis et~al\mbox{.}}{2015}]{Batoulis2015}
{\sc Batoulis, K.}, {\sc Meyer, A.}, {\sc Bazhenova, E.}, {\sc Decker, G.},
  {\sc and} {\sc Weske, M.} 2015.
\newblock Extracting decision logic from process models.
\newblock In {\em Proc.\ of the 27th Int.\ Conf.\ on Advanced Information
  Systems Engineering (CAiSE)}. Springer.

\bibitem[\protect\citeauthoryear{Calvanese, Dumas, Laurson, Maggi, Montali, and
  Teinemaa}{Calvanese et~al\mbox{.}}{2016}]{CDL16}
{\sc Calvanese, D.}, {\sc Dumas, M.}, {\sc Laurson, {\"U}.}, {\sc Maggi,
  F.~M.}, {\sc Montali, M.}, {\sc and} {\sc Teinemaa, I.} 2016.
\newblock Semantics and analysis of {DMN} decision tables.
\newblock In {\em Proc.\ of the 14th Int.\ Conf.\ on Business Process
  Management (BPM)}. Lecture Notes in Computer Science, vol. 9850. Springer,
  217--233.

\bibitem[\protect\citeauthoryear{Calvanese, Dumas, Maggi, and
  Montali}{Calvanese et~al\mbox{.}}{2017}]{CDMM17}
{\sc Calvanese, D.}, {\sc Dumas, M.}, {\sc Maggi, F.~M.}, {\sc and} {\sc
  Montali, M.} 2017.
\newblock Semantic {DMN}: {Formalizing} decision models with domain knowledge.
\newblock In {\em Proc.\ of the 1st Int.\ Joint Conf.\ on Rules and Reasoning
  (RuleML+RR)}. Lecture Notes in Computer Science, vol. 10364. Springer,
  70--86.

\bibitem[\protect\citeauthoryear{{CODASYL Decision Table Task Group}}{{CODASYL
  Decision Table Task Group}}{1982}]{CodasylReport}
{\sc {CODASYL Decision Table Task Group}}. 1982.
\newblock {\em A Modern Appraisal of Decision Tables: a {CODASYL} Report}.
\newblock ACM.

\bibitem[\protect\citeauthoryear{Drabent, Eiter, Ianni, Krennwallner,
  Lukasiewicz, and Maluszynski}{Drabent et~al\mbox{.}}{2009}]{DEIK09}
{\sc Drabent, W.}, {\sc Eiter, T.}, {\sc Ianni, G.}, {\sc Krennwallner, T.},
  {\sc Lukasiewicz, T.}, {\sc and} {\sc Maluszynski, J.} 2009.
\newblock Hybrid reasoning with rules and ontologies.
\newblock In {\em Semantic Techniques for the Web, The {REWERSE} Perspective},
  {F.~Bry} {and} {J.~Maluszynski}, Eds. Lecture Notes in Computer Science, vol.
  5500. Springer, 1--49.

\bibitem[\protect\citeauthoryear{Eiter, Kaminski, Redl, Sch{\"u}ller, and
  Weinzierl}{Eiter et~al\mbox{.}}{2017}]{EKRS17}
{\sc Eiter, T.}, {\sc Kaminski, T.}, {\sc Redl, C.}, {\sc Sch{\"u}ller, P.},
  {\sc and} {\sc Weinzierl, A.} 2017.
\newblock Answer set programming with external source access.
\newblock In {\em Reasoning Web: Semantic Interoperability on the Web -- 13th
  Int.\ Summer School Tutorial Lectures (RW)}. Lecture Notes in Computer
  Science, vol. 10370. Springer, 204--275.

\bibitem[\protect\citeauthoryear{Eiter, Lutz, Ortiz, and Simkus}{Eiter
  et~al\mbox{.}}{2009}]{ELOS09b}
{\sc Eiter, T.}, {\sc Lutz, C.}, {\sc Ortiz, M.}, {\sc and} {\sc Simkus, M.}
  2009.
\newblock Query answering in description logics: The {Knots} approach.
\newblock In {\em Proc.\ of the 16th Int.\ Workshop on Logic, Language,
  Information and Computation (WoLLIC)}. Lecture Notes in Computer Science,
  vol. 5514. Springer, 26--36.

\bibitem[\protect\citeauthoryear{Enderton}{Enderton}{2001}]{Ende01}
{\sc Enderton, H.~B.} 2001.
\newblock {\em A Mathematical Introduction to Logic\/}, 2nd ed.
\newblock Academic Press.

\bibitem[\protect\citeauthoryear{Haarslev, M{\"o}ller, and Wessel}{Haarslev
  et~al\mbox{.}}{2001}]{HaMW01}
{\sc Haarslev, V.}, {\sc M{\"o}ller, R.}, {\sc and} {\sc Wessel, M.} 2001.
\newblock The description logic $\mathsf{ALCNH}_{R+}$ extended with concrete
  domains: A practically motivated approach.
\newblock In {\em Proc.\ of the 1st Int.\ Joint Conf.\ on Automated Reasoning
  (IJCAR)}. 29--44.

\bibitem[\protect\citeauthoryear{Hoover and Chen}{Hoover and
  Chen}{1995}]{HoCh95}
{\sc Hoover, D.~N.} {\sc and} {\sc Chen, Z.} 1995.
\newblock Tablewise, a decision table tool.
\newblock In {\em Proc.\ of the 10th Annual Conf.\ on Computer Assurance
  Systems Integrity, Software Safety and Process Security (COMPASS)}. IEEE
  Computer Society Press, 97--108.

\bibitem[\protect\citeauthoryear{Horrocks, Kutz, and Sattler}{Horrocks
  et~al\mbox{.}}{2006}]{HoKS06}
{\sc Horrocks, I.}, {\sc Kutz, O.}, {\sc and} {\sc Sattler, U.} 2006.
\newblock The even more irresistible $\mathsf{SROIQ}$.
\newblock In {\em Proc.\ of the 10th Int.\ Conf.\ on the Principles of
  Knowledge Representation and Reasoning (KR)}. 57--67.

\bibitem[\protect\citeauthoryear{Horrocks and Sattler}{Horrocks and
  Sattler}{2001}]{HoSa01}
{\sc Horrocks, I.} {\sc and} {\sc Sattler, U.} 2001.
\newblock Ontology reasoning in the $\mathsf{SHOQ}${(D)} description logic.
\newblock In {\em Proc.\ of the 17th Int.\ Joint Conf.\ on Artificial
  Intelligence (IJCAI)}. 199--204.

\bibitem[\protect\citeauthoryear{Krisnadhi, Maier, and Hitzler}{Krisnadhi
  et~al\mbox{.}}{2011}]{KrMH11}
{\sc Krisnadhi, A.}, {\sc Maier, F.}, {\sc and} {\sc Hitzler, P.} 2011.
\newblock {OWL} and rules.
\newblock In {\em Reasoning Web: Semantic Technologies for the Web of Data --
  7th Int.\ Summer School Tutorial Lectures (RW)}. Lecture Notes in Computer
  Science, vol. 6848. Springer, 382--415.

\bibitem[\protect\citeauthoryear{Lutz}{Lutz}{2002a}]{Lutz02c}
{\sc Lutz, C.} 2002a.
\newblock The complexity of reasoning with concrete domains.
\newblock Ph.D. thesis, Teaching and Research Area for Theoretical Computer
  Science, RWTH Aachen.

\bibitem[\protect\citeauthoryear{Lutz}{Lutz}{2002b}]{Lutz02d}
{\sc Lutz, C.} 2002b.
\newblock Description logics with concrete domains -- a survey.
\newblock In {\em Proc.\ of the 4th Conf.\ on Advances in Modal Logic
  (AiML~2012)}. 265--296.

\bibitem[\protect\citeauthoryear{Motik and Horrocks}{Motik and
  Horrocks}{2008}]{MoHo08}
{\sc Motik, B.} {\sc and} {\sc Horrocks, I.} 2008.
\newblock {OWL} datatypes: Design and implementation.
\newblock In {\em Proc.\ of the 7th Int.\ Semantic Web Conf.\ (ISWC)}. Lecture
  Notes in Computer Science, vol. 5318. Springer, 307--322.

\bibitem[\protect\citeauthoryear{Motik, Parsia, and Patel-Schneider}{Motik
  et~al\mbox{.}}{2012}]{W3Crec-OWL2-Syntax}
{\sc Motik, B.}, {\sc Parsia, B.}, {\sc and} {\sc Patel-Schneider, P.~F.} 2012.
\newblock {OWL~2} {W}eb {O}ntology {L}anguage structural specification and
  functional-style syntax (second edition).
\newblock {W3C} {R}ecommendation, World Wide Web Consortium. Dec.
\newblock Available at \protect\url{http://www.w3.org/TR/owl2-syntax/}.

\bibitem[\protect\citeauthoryear{Motik and Rosati}{Motik and
  Rosati}{2010}]{MoRo10}
{\sc Motik, B.} {\sc and} {\sc Rosati, R.} 2010.
\newblock Reconciling description logics and rules.
\newblock {\em J.\ of the ACM\/}~{\em 57,\/}~5, 30:1--30:62.

\bibitem[\protect\citeauthoryear{N{\'e}meti}{N{\'e}meti}{1986}]{Neme86}
{\sc N{\'e}meti, I.} 1986.
\newblock Free algebras and decidability in algebraic logic.
\newblock Ph.D. thesis, Mathematical Institute of The Hungarian Academy of
  Sciences, Budapest.

\bibitem[\protect\citeauthoryear{{OMG}}{{OMG}}{2016}]{DMN}
{\sc {OMG}}. 2016.
\newblock {Decision Model and Notation (DMN)~1.1}.
\newblock Available at \protect\url{http://www.omg.org/spec/DMN/1.1/}.

\bibitem[\protect\citeauthoryear{Ortiz}{Ortiz}{2010}]{Orti10}
{\sc Ortiz, M.} 2010.
\newblock Query answering in expressive description logics: Techniques and
  complexity results.
\newblock Ph.D. thesis, Vienna University of Technology.

\bibitem[\protect\citeauthoryear{Ortiz, Simkus, and Eiter}{Ortiz
  et~al\mbox{.}}{2008}]{OrSE08}
{\sc Ortiz, M.}, {\sc Simkus, M.}, {\sc and} {\sc Eiter, T.} 2008.
\newblock Worst-case optimal conjunctive query answering for an expressive
  description logic without inverses.
\newblock In {\em Proc.\ of the 23rd AAAI Conf.\ on Artificial Intelligence
  (AAAI)}. AAAI Press, 504--510.

\bibitem[\protect\citeauthoryear{Pan and Horrocks}{Pan and
  Horrocks}{2003}]{PaHo03}
{\sc Pan, J.~Z.} {\sc and} {\sc Horrocks, I.} 2003.
\newblock Web ontology reasoning with datatype groups.
\newblock In {\em Proc.\ of the 2nd Int.\ Semantic Web Conf.\ (ISWC)}. Lecture
  Notes in Computer Science, vol. 2870. Springer, 47--63.

\bibitem[\protect\citeauthoryear{Pawlak}{Pawlak}{1987}]{Pawlak1987}
{\sc Pawlak, Z.} 1987.
\newblock Decision tables -- {A} rough set approach.
\newblock {\em Bull.\ of the EATCS\/}~{\em 33}, 85--95.

\bibitem[\protect\citeauthoryear{Pooch}{Pooch}{1974}]{Pooch74}
{\sc Pooch, U.~W.} 1974.
\newblock Translation of decision tables.
\newblock {\em ACM Computing Surveys\/}~{\em 6,\/}~2, 125--151.

\bibitem[\protect\citeauthoryear{Savkovic and Calvanese}{Savkovic and
  Calvanese}{2012}]{SaCa12}
{\sc Savkovic, O.} {\sc and} {\sc Calvanese, D.} 2012.
\newblock Introducing datatypes in \textit{DL-Lite}.
\newblock In {\em Proc.\ of the 20th Eur.\ Conf.\ on Artificial Intelligence
  (ECAI)}. Frontiers in Artificial Intelligence and Applications, vol. 242. IOS
  Press, 720--725.

\bibitem[\protect\citeauthoryear{Shearer, Motik, and Horrocks}{Shearer
  et~al\mbox{.}}{2008}]{ShMH08}
{\sc Shearer, R.}, {\sc Motik, B.}, {\sc and} {\sc Horrocks, I.} 2008.
\newblock {HermiT}: A highly-efficient {OWL} reasoner.
\newblock In {\em Proc.\ of the 5th Int.\ Workshop on OWL: Experiences and
  Directions (OWLED)}. CEUR Workshop Proceedings,
  {\upshape\protect\url{http://ceur-ws.org/}}, vol. 432.

\bibitem[\protect\citeauthoryear{Sirin and Parsia}{Sirin and
  Parsia}{2006}]{SiPa06}
{\sc Sirin, E.} {\sc and} {\sc Parsia, B.} 2006.
\newblock Pellet system description.
\newblock In {\em Proc.\ of the 19th Int.\ Workshop on Description Logics
  (DL)}. CEUR Workshop Proceedings,
  {\upshape\protect\url{http://ceur-ws.org/}}, vol. 189.

\bibitem[\protect\citeauthoryear{Tsarkov and Horrocks}{Tsarkov and
  Horrocks}{2006}]{TsHo06}
{\sc Tsarkov, D.} {\sc and} {\sc Horrocks, I.} 2006.
\newblock \textsf{FaCT++} description logic reasoner: System description.
\newblock In {\em Proc.\ of the 3rd Int.\ Joint Conf.\ on Automated Reasoning
  (IJCAR)}. 292--297.

\bibitem[\protect\citeauthoryear{Vanthienen and Dries}{Vanthienen and
  Dries}{1993}]{VanD93}
{\sc Vanthienen, J.} {\sc and} {\sc Dries, E.} 1993.
\newblock Illustration of a decision table tool for specifying and implementing
  knowledge based systems.
\newblock In {\em Proc.\ of the 5th IEEE Int.\ Conf.\ on Tools with Artificial
  Intelligence (ICTAI)}. IEEE Computer Society Press, 198--205.

\bibitem[\protect\citeauthoryear{Vanthienen and Dries}{Vanthienen and
  Dries}{1994}]{Vanthienen1994}
{\sc Vanthienen, J.} {\sc and} {\sc Dries, E.} 1994.
\newblock Illustration of a decision table tool for specifying and implementing
  knowledge based systems.
\newblock {\em Int.\ J.\ on Artificial Intelligence Tools\/}~{\em 3,\/}~2,
  267--288.

\bibitem[\protect\citeauthoryear{Vanthienen, Mues, and Aerts}{Vanthienen
  et~al\mbox{.}}{1998}]{Vanthienen1998}
{\sc Vanthienen, J.}, {\sc Mues, C.}, {\sc and} {\sc Aerts, A.} 1998.
\newblock An illustration of verification and validation in the modelling phase
  of {KBS} development.
\newblock {\em Data and Knowledge Engineering\/}~{\em 27,\/}~3, 337--352.

\bibitem[\protect\citeauthoryear{Zaidi and Levis}{Zaidi and
  Levis}{1997}]{ZaLe97}
{\sc Zaidi, A.~K.} {\sc and} {\sc Levis, A.~H.} 1997.
\newblock Validation and verification of decision making rules.
\newblock {\em Automatica\/}~{\em 33,\/}~2, 155--169.

\end{thebibliography}

% \appendix

\label{lastpage}
\end{document}